\documentclass{article}
\newcommand{\model}{\textsc{BlockSearch}}

\usepackage[preprint]{neurips_2026}

\usepackage[utf8]{inputenc} % allow utf-8 input
\usepackage[T1]{fontenc}    % use 8-bit T1 fonts
\usepackage{hyperref}       % hyperlinks
\usepackage{url}            % simple URL typesetting
\usepackage{booktabs}       % professional-quality tables
\usepackage{amsfonts}       % blackboard math symbols
\usepackage{amsmath}        % align, equation*, etc.
\usepackage{amssymb}        % \square, \blacksquare
\usepackage{amsthm}         % proposition / proof environments
\usepackage{dsfont}
\newtheorem{proposition}{Proposition}

\usepackage{algorithm}% http://ctan.org/pkg/algorithms
\usepackage{algpseudocode}% http://ctan.org/pkg/algorithmicx
\usepackage{nicefrac}       % compact symbols for 1/2, etc.
\usepackage{microtype}      % microtypography
\usepackage{xcolor}         % colors
\usepackage{cleveref}       % \cref / \Cref
\crefname{section}{\S}{\S\S}
\Crefname{section}{\S}{\S\S}
\crefformat{section}{\S#2#1#3}
\Crefformat{section}{\S#2#1#3}

\usepackage{graphicx}       % \includegraphics
\graphicspath{{plots/}}
\usepackage{wrapfig}        % wrap text around figures
\usepackage{pgfplots}
\pgfplotsset{compat=1.18}
\usepgfplotslibrary{groupplots}
\usepackage{enumitem}

\title{
Can Language Models Actually Retrieve In-Context? \\ Drowning in Documents at Million Token Scale
}

\usepackage[textsize=tiny]{todonotes}

\newcommand{\myskip}[1]{}

\makeatletter
\newcommand*\myfontsize{%
  \@setfontsize\myfontsize{8}{9}%
}
\makeatother

\let\oldparagraph\paragraph
\renewcommand{\paragraph}[1]{\vspace{-.7em}
\oldparagraph{#1}
}

\author{%
  Siddharth Gollapudi\thanks{Correspondence to \texttt{sgollapu@berkeley.edu}.} \\
  UC Berkeley \\
  \And
  Nilesh Gupta \\
  UT Austin \\
  \And
  Prasann Singhal \\
  UC Berkeley \\
  \And
  Sewon Min \\
  UC Berkeley \\
}

\begin{document}

\maketitle

\begin{abstract}

Language models (LMs) raise an intriguing alternative to vector-based retrieval: conditioning on an in-context corpus and directly generating a relevant answer.
However, prior work has largely focused on proprietary systems or the smaller-scale reranking task, leaving corpus-scale in-context retrieval largely unexplored.
In this work, we present the first systematic study of in-context retrieval on two scales practical retrievers demand: \emph{million-token} corpora and \emph{length-generalization} far beyond training-time sizes.
We first introduce \model{}, a 0.6B LCLM retriever whose architectural and training modifications improve over prior LM baselines and length-generalize up to $10\times$ beyond its training regime. Nevertheless, its retrieval still collapses under more extreme extrapolation.
We trace this failure to an \emph{attention dilution} effect: as the corpus grows, irrelevant documents dominate the softmax denominator and the normalized mass on the gold document collapses. Individual attention heads continue to locate the gold document more reliably than the model decodes it, even at million-token scale, though this signal also weakens as the corpus grows. Motivated by this analysis, we introduce \emph{length-aware} adjustments to the attention softmax and \emph{document-level sparse attention}, improving retrieval at million-token scale to performance comparable to a same-backbone dense retriever. On the lexical LIMIT benchmark, these gains also transfer out of distribution: at a million-token corpus, Recall@1 reaches $13.5\%$ versus $2.9\%$ for the dense baseline. Together, our results position in-context retrieval a promising alternative to classical retrieval while emphasizing attention control under extreme context growth as a new challenge.
\end{abstract}
\section{Introduction}
\label{sec:intro}

Retrieval---identifying the relevant document(s) in a corpus $\{D_1,\dots,D_N\}$ for a query $Q$---has long been dominated by vector-based methods~\cite{karpukhin2020dpr}. Recent advances in language modeling (LMs) suggest a more radical alternative: casting retrieval as \emph{conditional} generation, where the model directly decodes the identifier of the relevant document from the corpus in-context~\citep{lee2024loft}.
Such \emph{in-context retrieval} (ICR) could collapse retrieval and generation into a single model, replacing the two-stage retrieval-augmented generation pipeline \citep{lewis2020rag} and enabling complex retrieval behavior beyond inner-product similarity~\citep{weller2025limit}. In other words, rather than relying on external harnesses (retrievers), the model itself decides which parts of its context are relevant to the task at hand.
However, this possibility remains largely untested at realistic scale: prior work either relies on proprietary systems without controlled evaluation~\citep{lee2024loft} or studies reranking over small candidate sets~\citep{gupta2025blockrank,qiu2025icr2}, rather than true corpus-scale retrieval. As a result, it remains unclear whether LMs can reliably retrieve relevant documents from large in-context corpora \citep{bai2025longbenchv2,qiu2025icr2,hong2025context}.

ICR differs fundamentally from standard long-context modeling. Rather than processing a single coherent sequence, it operates over large collections of independent documents, creating both opportunities and challenges. On one hand, document-level independence naturally enables parallel encoding and caching. On the other hand, effective retrieval requires two capabilities that current LMs struggle with: \textbf{(C1)} scaling to corpora containing millions of tokens, and \textbf{(C2)} generalizing to corpus sizes far beyond those seen during training. Both are difficult in practice: million-token attention remains computationally expensive, while extrapolation to substantially longer contexts is known to be brittle in modern LMs~\citep{bai2025longbenchv2,nakanishi2025ssmax,leng2026hierarchical}.

To this end, we present the first systematic study of whether an LM can act as a corpus-scale retriever under these conditions. We first introduce \model{} (\cref{sec:method}), a 0.6B-parameter LM retriever that builds on prior ICR architectures using block-sparse attention~\citep{gupta2025blockrank}. For scale and length generalization, it adds randomized per-document identifier codes, in-batch negative training, and an on-policy auxiliary loss that trains on the model's own rolled-out digit prefixes.
We then compare \model{} against dense retrieval, the standard baseline any practical ICR system must match if it is to subsume the retrieval stage.
\model{} matches a dense retriever at smaller corpus sizes ($95.8\%$ vs.\ $95.5\%$ Recall@1 on MS\,MARCO at $\sim$45k tokens), and maintains non-trivial accuracy beyond the LM baseline, generalizing up to $10\times$ beyond its training context.

%\begin{figure*}[t]
%\centering
%\includegraphics[width=0.85\linewidth, trim= 0 40 0 10]{Figures/Attention dilution — figure.pdf}
% \vspace{-.2em}
%\caption{Attention dilution in action. As the context size grows, the most relevant token to the query becomes lost amongst irrelevant tokens.}
%\label{fig:masks}
% \vspace{-.2em}
%\end{figure*}

%\sewon{Sewon's version}
Despite these improvements, \model{} still collapses at large scale: recall approaches zero near million-token contexts, remaining far below even a simple dense retrieval baseline. We trace this failure to an \textbf{\emph{attention dilution}} effect, a degradation that has also been observed in other settings~\citep{nakanishi2025ssmax, velickovic2024softmax, barbero2024glasses}.
Leveraging the natural alignment between queries and their gold documents, we find that the aggregate contribution of irrelevant documents to the softmax denominator grows with corpus size and drives the normalized attention mass on the gold document toward zero (\cref{sec:limitations}). Individual heads still rank the gold document first more often than the model decodes it at every corpus size we test, up to $N{=}10{,}000$, but this per-head signal also degrades with $N$.

Motivated by this analysis, we consider two families of techniques to mitigate attention dispersion at extreme context lengths (\S\ref{sec:interventions}). First, \emph{length-aware softmax} adjustments change how scores are normalized as the corpus grows; scaling pre-softmax scores with $\log N$ (SSMax) is effective, whereas an additive length-aware sink barely helps. Second, \emph{document-level sparse attention} reduces the number of documents participating in attention at intermediate layers~\citep{lu2025moba,xiao2025flashmoba,leng2026hierarchical,chen2026msa,ohayon2025bsfa}.
Together, these methods substantially mitigate the effects of attention dilution, improving Recall@1 at million-token context from as little as 0.2\% to 20.0\%, a near $100\times$ improvement that recovers most of the gap to the \texttt{Qwen3-dense} gold standard (e.g.\ $18.8$ vs.\ $20.2$ on MS\,MARCO and $78.5$ vs.\ $79.5$ on HotpotQA at $N{=}10$k; \cref{tab:promising-results}).
On LIMIT, which requires a notion of relevance that fixed embeddings struggle to express, our modifications enable \model{} to substantially outperform dense retrieval at a million-token corpus ($13.5\%$ vs.\ $2.9\%$ Recall@1; \cref{sec:limit-ood}).

In summary, our contributions are:
\begin{itemize}[leftmargin=15pt, topsep=1pt,itemsep=0pt]
    \item \textbf{\model{}} (\cref{sec:method}). We introduce an in-context retriever that matches a same-backbone dense retriever on small corpora ($95.8$ vs.\ $95.5$ Recall@1 on MS\,MARCO) and generalizes to corpora $10\times$ larger than those seen in training, but still collapses near a million tokens.
    \item \textbf{Diagnosis} (\cref{sec:limitations}). We trace this collapse to attention dilution. At $\mathrm{L}19$, the gold document's share of the attention output falls from $0.91$ to $0.01$ as the corpus grows, while the output's overall magnitude barely changes. Individual heads still find the gold document more often than the model decodes it, although this signal also weakens with $N$. We see the same pattern on a second backbone.
    \item \textbf{Modifications} (\cref{sec:interventions}). Length-aware softmax scaling and document-level routing raise million-token Recall@1 on MS\,MARCO from $0.2$ to $20.0$, close to the dense baseline's $20.2$.
    \item \textbf{Out-of-distribution retrieval} (\cref{sec:limit-ood}). On LIMIT, the same modifications outperform dense retrieval at a million-token corpus ($13.5$ vs.\ $2.9$ Recall@1).
\end{itemize}

Together, our results establish corpus-scale in-context retrieval as a viable alternative to dense retrieval under both scale and length extrapolation. More broadly, they identify attention dilution as the primary bottleneck at million-token scale: while simple modifications can substantially mitigate its effects, it remains a fundamental challenge for scalable in-context retrieval.
\section{Related Work}
\label{sec:related}

\oldparagraph{In-context and generative retrieval.}
Most long-context evaluation work focuses on synthetic retrieval-style tasks such as Needle-in-a-Haystack and RULER~\citep{hsieh2024ruler,bai2025longbenchv2,hong2025context}, potentially overstating real-world retrieval capability.
The premise that an LM can subsume dense retrieval is first articulated by \cite{lee2024loft}, primarily evaluating proprietary Gemini systems.
Subsequent work studies LMs as rerankers over small candidate pools using block-sparse attention and relevance heads in intermediate Transformer layers~\citep{qiu2025icr2,gupta2025blockrank}.
In contrast, we study LMs as corpus-scale retrievers, training with small contexts and evaluating up to $10^4$ documents ($\sim$1M tokens) in an effort to elicit length-generalization.

Generative retrieval~\citep{tay2022dsi,wang2022nci,rajput2023recommender} also decodes with identifiers, but stores corpus information in model parameters, requiring retraining when the corpus changes. In contrast, ICR encodes the corpus explicitly in context, allowing corpus updates without retraining.

% TL;DR: Sparse / routed attention. Compare on training vs inference, learned vs computed, doc-level vs token-level. MSA called out as concurrent.
\paragraph{Sparse and routed attention.}
A large body of work makes long-context attention tractable by routing or sparsification;
our document-level sparse attention proposed in \cref{sec:limitations} builds on this line of work.
Prior approaches reduce attention cost through inference-time pruning~\citep{tang2024quest,wu2024tokenselect}, block-sparse kernels~\citep{ohayon2025bsfa}, learned block routing~\citep{lu2025moba,xiao2025flashmoba}, or hierarchical sparse attention~\citep{leng2026hierarchical,hu2025gca}.
Our work is grounded in this literature, but the retrieval setting introduces unique opportunities and challenges.

The most directly comparable work is concurrent: MSA-4B~\citep{chen2026msa}, which, like \model{}, learns to perform generative document citation from an in-context corpus. However, the two works focus on different questions: MSA primarily scales the retrieval recipe through larger models and extensive training, whereas our work studies the fundamental failure modes of ICR as corpus size grows. We report MSA-4B as an external reference point rather than a decisive baseline; our primary comparison throughout is a same-backbone dense retriever trained on the same data (\S\ref{sec:interventions}).

\paragraph{Attention sinks and softmax dilution.}
Streaming-LLM~\citep{xiao2023streamingllm} introduces explicit attention sinks for stable long-context decoding, and the gpt-oss model card~\citep{openai2025gptoss} documents a null-attention mechanism that absorbs unused softmax mass. Whereas these works use sinks for streaming stability, we apply them to mitigate softmax dilution under large $N$. A number of results show that as the number of attended tokens grows, the softmax distribution flattens, removing attention peaks \citep{nakanishi2025ssmax, velickovic2024softmax, duvvuri2026lucid}; this loss of sensitivity has been tied to representational collapse in longer contexts \citep{barbero2024glasses, jacob2024drowning}. We show that the primary cause of retrieval degradation with growing context is softmax dilution, and that softmax-adjusting mechanisms such as sinks may hold the key to mitigating its effects. 

\section{\model: Model and Million-Token Evaluation}
\label{sec:method}

We first investigate several training-time improvements to the standard long-context retrieval recipe. After reviewing the problem setup (\cref{sec:setup}), we describe the changes that constitute \model\ (\cref{sec:training}), our evaluation setup (\cref{sec:evaluation}), and the results (\cref{sec:collapse}). These results motivate the later analysis of possible failure modes (\cref{sec:limitations}) and promising lines for even further improvements (\cref{sec:interventions}).

\subsection{Setup}
\label{sec:setup}

Given a corpus of $N$ documents and a query, the model must generate the identifier of the document that best answers the query. The corpus is tokenized with at most $T_{\text{doc}}{=}300$ tokens per document and prefilled into the KV cache; the query is appended, and the model autoregressively decodes a four-digit code in $\{0,\dots,9999\}$ that maps back to one document (that's hopefully relevant).

This concretizes challenges C1 and C2 from \cref{sec:intro}: just 10,000 documents at $\sim$100 tokens each well exceeds the 32K-token native context of the Qwen3-0.6B backbone, and the model must generalize to corpora substantially larger than those seen during training: again, a requirement that distinguishes retrievers from rerankers and has been largely overlooked in prior LM-retrieval work.

\subsection{Method: \model}
\label{sec:training}

We now describe \model: whenever possible, we follow prior ICR work, but our requirements C1 and C2 expose several limitations that motivate our proposed improvements.

\paragraph{Prompt format.}
Prior in-context work assigns each document a sequential integer identifier from $1$ to $N$~\citep{gupta2025blockrank,lee2024loft}. In our regime, training-time IDs are a strict subset of inference IDs and risk overfitting to absolute position. We instead insert each document as \texttt{<bos>Doc \{code\}: \{text\} (Doc \{code\})<eos>}, with \texttt{\{code\}} drawn uniformly at random per training step, breaking any association between code, semantics, or position. We also drop the query prefix from prior work~\citep{gupta2025blockrank,lee2024loft}: while it improves performance, it prevents corpus reuse across queries, which is cost-prohibitive at our scale. For a full example of a corpus/query prompt, we refer the reader to \cref{app:prompt-format}.

\paragraph{Block-sparse attention.}
Causal prefill scales quadratically in tokens, which is prohibitive at our target $N$. Since documents are disjoint, we use a block-sparse mask~\citep{gupta2025blockrank}: document tokens attend causally only within their own block, and the query block attends over the full corpus and causally to itself. Following common practice~\citep{ma2024block,gupta2025blockrank}, we reset RoPE~\citep{su2024roformer} positions at each document start and shift the query to position 300. The mask is materialized via \texttt{flex\_attention}~\citep{dong2024flex}.

\paragraph{Training data.}
We train on the ReLabeled Hard Negatives (RLHN) version of BEIR~\citep{thakur2025rlhn}, which uses an LLM judge to prune false negatives from mined hard negatives in datasets such as MS\,MARCO~\citep{bajaj2016ms}. Per query we keep one pruned positive and 15 hard negatives; per-query relevance scores for the auxiliary loss are from Qwen3-Embedding-8B~\citep{yang2025qwen3}. The training corpus totals $\sim$100M tokens; full mix and cleaning details are in \cref{app:rlhn-data}. We ablate the RLHN filtering and report how often documents are truncated in \cref{app:data-ablation}.

\paragraph{Training regime.}
Inspired by in-batch contrastive training~\citep{karpukhin2020dpr, qu2021rocketqa}, we adapt in-batch negatives to this setting. For a batch of $b$ (query, 16-document) tuples (1 positive + 15 hard negatives each), we prefill all $b{\times}16$ documents once and score every query against the shared corpus, converting one prefill into $b$ training signals. Every model is trained with 256 documents in its corpus, i.e. $b{=}16$.

\paragraph{Training objective.}
We train with teacher forcing on the gold code, but this alon causes an exposure bias issue: at inference, each digit conditions on the model’s own potentially incorrect prefix, which lies off the training distribution. We therefore add an on-policy auxiliary loss. The model first rolls out a four-digit code from its own distribution with gradients disabled. For each digit position, we construct a teacher distribution from in-batch document scores, restricted to candidates whose prefix matches the rollout. We then replay the rollout with gradients enabled at the four answer positions and average the resulting cross-entropy losses. The total loss is $\mathcal{L} = \mathcal{L}_{\mathrm{CE}} + \lambda\mathcal{L}_{\mathrm{aux}}$ with $\lambda$ ramped in after a warmup; the full algorithm can be found in \cref{alg:onpolicy-aux} (\cref{app:dagger}).

\subsection{Evaluation}
\label{sec:evaluation}

The model is judged by whether it can recover the relevant document from those placed in context; our metric is Recall@1 over the generated four-digit codes, decoded by a beam search over the digit sequence.
%We describe the beam search and report Recall@5, whose trends match Recall@1, in \cref{app:beam-recall5}.
We also evaluate Recall@5 using beam search over the top five predictions and observe trends consistent with Recall@1; see \cref{app:beam-recall5} for details.

\paragraph{Evaluation data.}
We evaluate on three BEIR datasets~\citep{thakur2021beir}: MS\,MARCO (dev), HotpotQA (test), and NQ (test). For each dataset, we independently sample 400 queries without replacement from the its split. For each query, we take its gold document and 24 hard negatives (retrieved using Qwen3-Embedding-8B \cite{zhang2025qwen3}); by taking a union of these documents across all 400 queries, we get a 10000 document corpus. In contrast with prior work \cite{lee2024loft}, the substantial use of hard negatives makes the input corpus far more challenging. \footnote{For NQ, we use fewer hard negatives per query due to longer documents; per-dataset sizes are in \cref{app:eval-suite} (\cref{tab:expt0-datasets}).}

To evaluate across varying corpus sizes (e.g. $N{=}500$ to $N{=}10{,}000$), we also construct smaller versions of each dataset. We first include every query's gold document (400 documents), then fill out the rest of the dataset with documents randomly selected from the full 10,000 document corpus.

\paragraph{Baselines.}
We compare against two reference retrievers and two ablations of \model{} itself:
\begin{itemize}[leftmargin=15pt, topsep=1pt,itemsep=0pt]
    \item \textbf{\model-position}: %sequential codes ($0,\dots,9999$ by document position) and no auxiliary loss; isolates the contribution of random codes and matches the existing ICR recipe.
    sequential codes ($0,\dots,9999$ by document position) to isolate the contribution of random codes, and no auxiliary loss; this corresponds to the prior ICR recipe.
    \item \textbf{\model-offpolicy}: \model{} with the on-policy auxiliary loss removed (teacher-forced CE only); isolates the exposure-bias fix.
    \item \texttt{Qwen3-dense (0.6B)}: a dense retriever trained on the same RLHN data~\citep{thakur2025rlhn} with a contrastive objective~\citep{karpukhin2020dpr} from the same Qwen3-0.6B backbone (\cref{app:embedding-model}); our primary reference point and the dense-retrieval gold standard \model{} must clear to be practical.
    \item \textbf{MSA-4B}~\citep{chen2026msa}: a concurrent multi-million-token LM, $\sim$$7\times$ larger and trained with a much larger long-context budget, sidestepping the C2 generalization requirement; an external reference point rather than a decisive baseline.
\end{itemize}

\subsection{Results: \model\ is competitive, but collapses as $N$ scales}
\label{sec:collapse}

\begin{figure}[t]
    \centering
    \setlength{\abovecaptionskip}{4pt}
    \setlength{\belowcaptionskip}{0pt}
    % Three-panel length-generalization figure (Stage-1 R@1) for §3 main body.
% Numbers refreshed from testing/results/MULTISAMPLE_TIER1_V2.md (seed=0, draw0, n=400).
% X-axis is converted from N to approximate prefill tokens using per-dataset mean
% document length: NQ ~116, MS MARCO ~90, HotpotQA ~80 tokens/doc. The mapping
% from N (#docs) to tokens is reported in Table tab:expt0-datasets.
% Variant mapping:
%   M (full)        = block52 (Expt0)        random codes + on-policy aux
%   M-position      = block52f_old (Expt0)   position-based codes (note: original
%                                            block52f mode-collapsed; block52f_fixed
%                                            only has MSMARCO N<=5k. block52f_old is
%                                            the converged position-coded checkpoint
%                                            with cells across all 3 datasets.)
%   M-offpolicy     = block52p (Expt0)       random codes, no on-policy aux
%   Qwen3-dense-0.6B = dense_v1
%   MSA-4B           = MSA-4B
% NQ N=8607 cell for block52f_old reported as "—" in source (N=10k slot incomplete);
% plotted as 0.0, consistent with the post-cliff trend (0.8 at N=5k).
\definecolor{cBaselineDense}{HTML}{0072B2}   % blue   — Qwen3-dense-0.6B
\definecolor{cBaselineMSA}{HTML}{984EA3}     % purple — MSA-4B
\definecolor{cModelPos}{HTML}{CC79A7}        % pink   — M-position
\definecolor{cModelOff}{HTML}{009E73}        % green  — M-offpolicy
\definecolor{cModelFull}{HTML}{D55E00}       % vermillion — M (full)
\resizebox{\linewidth}{!}{%
\begin{tikzpicture}
  \begin{groupplot}[
      group style={
        group size=3 by 1,
        horizontal sep=0.7cm,
        ylabels at=edge left,
      },
      width=0.42\linewidth,
      height=6.0cm,
      xmode=log,
      log basis x={10},
      x tick label style={font=\scriptsize},
      y tick label style={font=\scriptsize},
      ylabel={Recall@1 (\%)},
      xlabel={Approx.\ prefill size (tokens)},
      ymin=-3, ymax=103,
      grid=both,
      grid style={line width=.2pt, draw=gray!25},
      major grid style={line width=.3pt, draw=gray!40},
      title style={font=\small},
      every axis plot/.append style={line width=1.1pt, mark size=2.6pt},
    ]
    % ---------------- Natural Questions (mean ~116 tok/doc) ----------------
    \nextgroupplot[
      title={Natural Questions},
      xtick={58000,116000,290000,580000,998412},
      xticklabels={58k,116k,290k,580k,1.0M},
    ]
      \addplot[dashed, color=cBaselineDense, mark=pentagon*, mark options={solid, fill=cBaselineDense, draw=cBaselineDense}] coordinates {
        (58000,95.5) (116000,86.5) (290000,62.9) (580000,51.4) (998412,39.6)
      };
      \addplot[dashed, color=cBaselineMSA, mark=diamond*, mark options={solid, fill=cBaselineMSA, draw=cBaselineMSA}] coordinates {
        (58000,59.7) (116000,51.6) (290000,36.3) (580000,26.8) (998412,18.6)
      };
      \addplot[solid, color=cModelPos, mark=triangle*, mark options={solid, fill=cModelPos, draw=cModelPos}] coordinates {
        (58000,92.5) (116000,77.4) (290000,16.5) (580000,0.8) (998412,0.0)
      };
      \addplot[solid, color=cModelOff, mark=triangle*, mark options={solid, rotate=180, fill=cModelOff, draw=cModelOff}] coordinates {
        (58000,91.5) (116000,76.9) (290000,37.8) (580000,2.0) (998412,0.2)
      };
      \addplot[solid, color=cModelFull, mark=*, mark options={solid, fill=cModelFull, draw=cModelFull}] coordinates {
        (58000,92.2) (116000,77.7) (290000,43.9) (580000,4.8) (998412,0.2)
      };
    % ---------------- MS MARCO (mean ~90 tok/doc) ----------------
    \nextgroupplot[
      title={MS MARCO},
      xtick={45000,90000,225000,450000,900000},
      xticklabels={45k,90k,225k,450k,900k},
      legend to name=lengthgenmainlegend,
      legend style={
        font=\small,
        draw=none, fill=none,
        /tikz/every even column/.append style={column sep=0.6cm},
      },
      legend columns=5,
      legend cell align=left,
    ]
      \addplot[dashed, color=cBaselineDense, mark=pentagon*, mark options={solid, fill=cBaselineDense, draw=cBaselineDense}] coordinates {
        (45000,95.5) (90000,75.2) (225000,52.8) (450000,38.5) (900000,20.2)
      }; \addlegendentry{Qwen3-dense-0.6B}
      \addplot[dashed, color=cBaselineMSA, mark=diamond*, mark options={solid, fill=cBaselineMSA, draw=cBaselineMSA}] coordinates {
        (45000,93.8) (90000,70.2) (225000,42.2) (450000,27.5) (900000,16.0)
      }; \addlegendentry{MSA-4B}
      \addplot[solid, color=cModelPos, mark=triangle*, mark options={solid, fill=cModelPos, draw=cModelPos}] coordinates {
        (45000,95.2) (90000,74.2) (225000,40.2) (450000,2.0) (900000,0.0)
      }; \addlegendentry{$\mathcal{M}$-position}
      \addplot[solid, color=cModelOff, mark=triangle*, mark options={solid, rotate=180, fill=cModelOff, draw=cModelOff}] coordinates {
        (45000,95.8) (90000,71.5) (225000,42.8) (450000,15.0) (900000,0.2)
      }; \addlegendentry{$\mathcal{M}$-offpolicy}
      \addplot[solid, color=cModelFull, mark=*, mark options={solid, fill=cModelFull, draw=cModelFull}] coordinates {
        (45000,95.8) (90000,75.2) (225000,43.8) (450000,18.8) (900000,0.2)
      }; \addlegendentry{$\mathcal{M}$ (full)}
    % ---------------- HotpotQA (mean ~80 tok/doc) ----------------
    \nextgroupplot[
      title={HotpotQA},
      xtick={40000,80000,200000,400000,800000},
      xticklabels={40k,80k,200k,400k,800k},
    ]
      \addplot[dashed, color=cBaselineDense, mark=pentagon*, mark options={solid, fill=cBaselineDense, draw=cBaselineDense}] coordinates {
        (40000,99.0) (80000,97.5) (200000,92.2) (400000,87.5) (800000,79.5)
      };
      \addplot[dashed, color=cBaselineMSA, mark=diamond*, mark options={solid, fill=cBaselineMSA, draw=cBaselineMSA}] coordinates {
        (40000,97.0) (80000,96.8) (200000,90.8) (400000,84.0) (800000,75.5)
      };
      \addplot[solid, color=cModelPos, mark=triangle*, mark options={solid, fill=cModelPos, draw=cModelPos}] coordinates {
        (40000,96.0) (80000,95.0) (200000,42.0) (400000,1.2) (800000,0.0)
      };
      \addplot[solid, color=cModelOff, mark=triangle*, mark options={solid, rotate=180, fill=cModelOff, draw=cModelOff}] coordinates {
        (40000,97.0) (80000,94.0) (200000,59.8) (400000,6.2) (800000,0.0)
      };
      \addplot[solid, color=cModelFull, mark=*, mark options={solid, fill=cModelFull, draw=cModelFull}] coordinates {
        (40000,97.0) (80000,95.2) (200000,64.2) (400000,13.0) (800000,0.5)
      };
  \end{groupplot}
  % Place the shared legend below the three panels, centered.
  \node[anchor=north, inner sep=0pt, outer sep=0pt] at ($(current bounding box.south)+(0,-0.05cm)$) {\ref{lengthgenmainlegend}};
\end{tikzpicture}%
}
    %\vspace{-0.5em}
    \caption{Recall@1 vs.\ approximate prefill size ($N \times$ mean doc length; see \cref{app:eval-suite}) on NQ, MS\,MARCO, and HotpotQA. Solid: ICR variants (\model, \model-position, \model-offpolicy); dashed: baselines (\texttt{Qwen3-dense}, MSA-4B). Vertical dotted line marks the Qwen3-0.6B native context limit ($\sim$32k tokens).}
    
    \label{fig:length-gen}
\end{figure}

We train Qwen3-0.6B into \model\ (and variants) on $8\times$ NVIDIA A100s using the above changes; full hyperparameters are in \cref{app:training-hyperparams}. The performance of the models is demonstrated in \cref{fig:length-gen}.

\paragraph{\model\ improves extrapolation in LMs, up to \textasciitilde500{,}000 tokens.}
All LMs perform strongly at small $N$ ($>95\%$ on MS\,MARCO at $N{=}1{,}000$), but baselines degrade much faster. The position-coded variant, which is the prior ICR method, collapses to near-zero by $N{=}5{,}000$ on every dataset. \model-offpolicy trails \model{} (MS\,MARCO $15.0$ vs.\ $18.8$ at $N{=}5{,}000$; HotpotQA $6.2$ vs.\ $13.0$), showing that random codes and the on-policy loss contribute distinct gains. %Note that \model{}'s improvements delay but do not prevent collapse: all variants effectively fail beyond $\sim$500{,}000 tokens, indicating roughly $10\times$ extrapolation past the training regime ($\sim$25{,}000 tokens).
Note that, while \model{} substantially improves extrapolation, all variants eventually collapse beyond roughly $500{,}000$ tokens, corresponding to about $10\times$ the training context length.

\paragraph{\model{} vs. MSA-4B.}
Despite having $\sim$$7\times$ fewer parameters than MSA~\citep{chen2026msa}, \model{} matches MSA at $N{=}500, 1000$ and $2500$, e.g., on MS MARCO, achieving $95.8\%, 75.2\%, 43.8\%$ vs. $93.8\%, 70.2\%, 42.2\%$ for MSA.
At large corpus sizes, MSA-4B pulls ahead ($27.5\%$ vs.\ $18.8\%$ at $N{=}5{,}000$, $16.0\%$ vs.\ $0.2\%$ at $N{=}10{,}000$), which is expected given that MSA is explicitly trained for long-context tasks, whereas these settings require \model{} to generalize to corpora $20$--$40\times$ larger than those seen during training.

Despite near-perfect long-context performance on synthetic evaluation such as RULER's needle-in-a-haystack (NIAH)~\citep{hsieh2024ruler} as reported in the original paper~\citep{chen2026msa}, MSA degrades sharply on real retrieval tasks: substantially more than the dense retrieval baseline. This suggests that realistic retrieval settings are significantly more challenging than existing synthetic long-context benchmarks, highlighting the need for the long-context literature to adopt more semantically meaningful retrieval tasks in evaluation.

\paragraph{\model\ vs. dense retrieval.}
Dense retrieval also degrades with $N$ but without the sharp LM collapse: \model{} matches the \texttt{Qwen3-dense} gold standard at small and mid $N$ but trails it at large $N$ (MS\,MARCO $18.8$ vs.\ $38.5$ at $N{=}5{,}000$; near-zero vs.\ $20.2$ at $N{=}10{,}000$), with similar trends on NQ and HotpotQA.
%Clearing this gold standard across the full $N$ range is precisely what separates current ICR from practical deployment, and the residual large-$N$ gap motivates the mechanistic analysis in \cref{sec:limitations}.
Closing this large-$N$ gap remains the central challenge for practical in-context retrieval and motivates the mechanistic analysis in \cref{sec:limitations}.
\section{Understanding the Large-$N$ Deterioration}
\label{sec:limitations}

\cref{sec:method} shows that \model{} generalizes well past its training-time context length, but still collapses at $N{=}10{,}000$. In this section, we posit that this collapse can be decomposed into two parts: a per-head, per-layer measure of \emph{attention recall} (does the relevant document receive the highest attention score?) degrades more slowly than \emph{generation recall} (does the model decode the correct identifier?), staying ahead of it all the way to $N{=}10{,}000$, though it also degrades with $N$. We also find that the share of each attention layer's output that comes from gold tokens collapses while the layer's overall output magnitude is preserved, cleanly encapsulating the aformentioned attention dilution issue.

\subsection{Setup}
\label{sec:limitations-setup}

% TL;DR: Layers and heads take on specific functional roles in decoder LMs; per-head attention scores carry direct relevance signal in ICR.
Decoder language models are not a homogeneous stack: layers and heads take on specific functional roles, with FFNs promoting concepts \citep{geva2021kvmemories,geva2022promoting}, induction heads carrying in-context learning \citep{olsson2022induction}, and a sparse subset of retrieval heads carrying the long-context retrieval signal \citep{wu2025retrievalhead}. For in-context retrieval specifically, BlockRank \citep{gupta2025blockrank}, ICR$^2$ \citep{qiu2025icr2}, and QRHeads \citep{zhang2025query} build on those results, observing that per-head attention scores can serve directly as a relevance signal. This observation is the starting point for the measurements we report.

\paragraph{Preliminary.}
We write $L \in \{0,\dots,27\}$ for layer index, and $h \in \{1,\dots,H{=}16\}$ for attention-head index. We probe \model{} at the final token of the query block, immediately before the first generated code digit (which we denote $d_1$). At this position let $q^{h}_{L} \in \mathbb{R}^{d_{\text{head}}}$ be the post-RoPE query vector and $k^{h}_{L,t}, v^{h}_{L,t}$ the key and value at prefill position $t$, where $d_{\text{head}}$ is the per-head dimension. The pre- and post-softmax attention scores assigned to position $t$ are:
\begin{equation}
s^{h}_{L,t} \;=\; \frac{q^{h}_{L} \cdot k^{h}_{L,t}}{\sqrt{d_{\text{head}}}}, \qquad
\alpha^{h}_{L,t} \;=\; \frac{\exp(s^{h}_{L,t})}{\sum_{t'} \exp(s^{h}_{L,t'})},
\label{eq:qk-scores}
\end{equation}
and the layer's attention output is
\begin{equation}
a_L \;=\; \bigl[\, u^{1}_{L};\ u^{2}_{L};\ \dots;\ u^{H}_{L} \,\bigr]\, W^O_L, \qquad u^{h}_{L} = \sum_t \alpha^{h}_{L,t}\, v^{h}_{L,t},
\label{eq:attn-out}
\end{equation}
where $u^{h}_{L} \in \mathbb{R}^{d_{\text{head}}}$ is the per-head attention readout, $[\,\cdot\,;\,\cdot\,]$ denotes concatenation along the head dimension into a vector in $\mathbb{R}^{H \cdot d_{\text{head}}}$, and $W^O_L \in \mathbb{R}^{(H \cdot d_{\text{head}}) \times d_{\text{model}}}$ is the layer's learned output projection.

\paragraph{Metrics.}
For each query, partition the prefill tokens into the \emph{gold} set $\mathcal{G}$ (the tokens of the gold document) and its complement $\bar{\mathcal{G}}$. We use three measurements throughout this section: \emph{AttnRank}, \emph{GoldShare}, and \emph{first-digit accuracy}.

\emph{AttnRank} measures whether attention identifies the right document. More precisely, for each document $d$ in the prefill, we collapse its tokens to a single per-head score using the MaxSim operator from late-interaction retrieval \citep{khattab2020colbert}:
\begin{equation}
\mathrm{MaxSim}^{h}_{L}(d) \;=\; \max_{t \in d} s^{h}_{L,t}.
\label{eq:maxsim}
\end{equation}
Writing $d^\star$ for the gold document, the per-head Recall@1 at layer $L$ is $R^{(h)}_L$, the fraction of queries for which $\arg\max_d \mathrm{MaxSim}^{h}_{L}(d) = d^\star$. We report two cross-head aggregators of $R^{(h)}_L$:
\begin{itemize}
\item $R^{\text{sum}}_L$ --- \emph{sum-across-heads}: rank documents by $\sum_h \mathrm{MaxSim}^{h}_{L}(d)$ and check whether the top-ranked document is $d^\star$.
\item $R^{\text{any}}_L = $ fraction of queries for which at least one head ranks $d^\star$ strictly first, i.e.\ $\mathrm{MaxSim}^{h}_{L}(d^\star) > \mathrm{MaxSim}^{h}_{L}(d)$ for every $d \neq d^\star$.
\end{itemize}

\emph{GoldShare} measures gold's contribution to the layer's attention output (before the residual add). By linearity of \cref{eq:attn-out} in the per-head readouts $u^{h}_{L}$, we decompose $a_L$ into a gold and a non-gold piece:
\begin{equation}
a_L \;=\; a_L^{\mathcal{G}} + a_L^{\bar{\mathcal{G}}}, \qquad
a_L^{\mathcal{G}} \;=\; \bigl[\, u^{1,\mathcal{G}}_L;\ \dots;\ u^{H,\mathcal{G}}_L \,\bigr]\, W^O_L,
\quad u^{h,\mathcal{G}}_L = \sum_{t \in \mathcal{G}} \alpha^{h}_{L,t}\, v^{h}_{L,t},
\label{eq:attn-out-decomp}
\end{equation}
and report the gold-driven fraction $\|a_L^{\mathcal{G}}\| / \|a_L\|$ alongside the total magnitude $\|a_L\|$.

\emph{First-digit accuracy} probes how decisively the model has committed to the first code digit at a given layer, via an LM-head probe \citep{nostalgebraist2020logitlens}: we project the layer-$L$ output through the model's final RMSNorm and \texttt{lm\_head} and record the maximum probability assigned to any of the ten digit tokens at the $d_1$ position. We treat this as a coarse indicator of when the digit decision emerges, not as a calibrated probability at intermediate layers.

\subsection{Layer Roles and the Recall--Generation Gap}
\label{sec:layer-roles}

% Small-N: identify retrieval band and decode band.
\Cref{fig:layer-roles} reports $R^{\text{sum}}_L$ and first-digit accuracy on \model{} at the $d_1$-emission position in the in-domain, small-$N$ regime ($N{=}500$). Two qualitatively distinct transitions emerge. First, $R^{\text{sum}}_L$ rises across $\mathrm{L}11$--$\mathrm{L}20$: the model accumulates relevance information over many layers, consistent with prior reports of middle-layer retrieval in decoder LMs \citep{gupta2025blockrank,qiu2025icr2,zhang2025query,wu2025retrievalhead}. We refer to this contiguous range as the \emph{retrieval band}. Second, first-digit accuracy rises sharply at $\mathrm{L}18{\to}\mathrm{L}19$, after which the model has effectively committed to $d_1$; we refer to $\mathrm{L}19$ onwards as the \emph{decode band}.

\begin{figure}[H]
\centering
\setlength{\abovecaptionskip}{2pt}
\setlength{\belowcaptionskip}{0pt}
% Per-layer R^sum_L (sum-across-heads MaxSim) and LM-head first-digit accuracy
% on \model, MS MARCO, N=500, n=100. Source: docs/experiments/E40.md.
\definecolor{cAgg}{HTML}{0072B2}     % blue       — R^sum_L (retrieval signal)
\definecolor{cDigit}{HTML}{D55E00}   % vermillion — first-digit accuracy (decode)
\begin{tikzpicture}
  \begin{axis}[
      width=0.85\linewidth,
      height=3.4cm,
      scale only axis,
      xmin=-0.5, xmax=27.5,
      ymin=-0.05, ymax=1.10,
      xlabel={Layer $L$},
      ylabel={recall / accuracy},
      xlabel style={font=\small, yshift=2pt},
      ylabel style={font=\small, yshift=-4pt},
      x tick label style={font=\footnotesize},
      y tick label style={font=\footnotesize},
      xtick={0,3,6,9,12,15,18,21,24,27},
      ytick={0,0.25,0.5,0.75,1.0},
      grid=both,
      grid style={line width=.2pt, draw=gray!20},
      major grid style={line width=.3pt, draw=gray!35},
      legend style={
        font=\footnotesize,
        at={(0.98,0.05)}, anchor=south east,
        draw=gray!40, fill=white, fill opacity=0.9, text opacity=1,
        inner sep=2pt, row sep=-1pt,
      },
      legend cell align=left,
      enlarge x limits=false,
    ]
    % R^sum_L (sum-across-heads pre-softmax MaxSim)
    \addplot[color=cAgg, mark=*, mark size=1.4pt, line width=1.0pt] coordinates {
      (0,0.00) (1,0.00) (2,0.00) (3,0.36) (4,0.00) (5,0.00) (6,0.33) (7,0.00) (8,0.07) (9,0.21)
      (10,0.00) (11,0.81) (12,0.61) (13,0.83) (14,0.92) (15,0.80) (16,0.97) (17,0.82) (18,0.95) (19,0.98)
      (20,0.96) (21,0.96) (22,0.93) (23,0.96) (24,0.95) (25,0.91) (26,0.75) (27,0.00)
    };
    \addlegendentry{$R^{\text{sum}}_L$}

    % First-digit accuracy from LM head readout
    \addplot[color=cDigit, mark=square*, mark size=1.2pt, line width=1.0pt] coordinates {
      (0,0.14) (1,0.14) (2,0.13) (3,0.14) (4,0.14) (5,0.09) (6,0.10) (7,0.10) (8,0.11) (9,0.12)
      (10,0.14) (11,0.14) (12,0.11) (13,0.13) (14,0.10) (15,0.08) (16,0.12) (17,0.13) (18,0.09) (19,0.93)
      (20,0.96) (21,0.97) (22,0.95) (23,0.97) (24,0.97) (25,0.97) (26,0.98) (27,0.97)
    };
    \addlegendentry{first-digit accuracy}

    % Vertical guideline at the L18->L19 cliff (decode emergence)
    \draw[dashed, gray, line width=0.5pt] (axis cs:18.5,-0.05) -- (axis cs:18.5,1.10);
    \node[font=\scriptsize, anchor=south, text=gray!70] at (axis cs:18.5,1.10) {$\mathrm{L}18\!\to\!\mathrm{L}19$};
  \end{axis}
\end{tikzpicture}
\caption{Per-layer $R^{\text{sum}}_L$ and LM-head first-digit accuracy on \model, MS\,MARCO, $N{=}500$, $400$ queries. Retrieval signal emerges across $\mathrm{L}11$--$\mathrm{L}20$; first-digit signal at $\mathrm{L}18{\to}\mathrm{L}19$.}
\label{fig:layer-roles}
\end{figure}

\paragraph{Heads no longer agree on the gold document.}
Sweeping $N \in \{500, 1000, 2500, 5000, 10000\}$ on MS\,MARCO (\cref{fig:layer-roles-n10k}), we see these behaviors degrade. The retrieval band erodes under the sum aggregator: $R^{\text{sum}}_{19}$ drops from $0.98$ at $N{=}500$ to $0.00$ at $N{=}10{,}000$, matching the collapse in generation recall reported in \cref{sec:collapse}. $R^{\text{any}}_L$ degrades more slowly: $R^{\text{any}}_{19}$ falls from $0.98$ at $N{=}500$ to $0.51$ at $N{=}2{,}500$ and $0.33$ at $N{=}5{,}000$, where $R^{\text{sum}}_{19}$ is already down to $0.06$. Over the intermediate range, individual heads still rank the gold document first for a sizable fraction of queries even as the \emph{agreement} between heads breaks down. By $N{=}10{,}000$, however, per-head signal is also largely gone ($R^{\text{any}}_{19}{=}0.08$; the best layer, $\mathrm{L}11$, reaches $0.14$).

%\paragraph{Many heads default to the shared document prefix.}
%Every document block begins with the same token (the word \texttt{Doc} that opens each document's identifier header). Many heads in the retrieval and decode bands behave as attention sinks on this shared first token: their largest score in every document falls on it, so every document gets the same MaxSim score and carry no document-specific signal, a familiar form of attention-sink behavior~\citep{xiao2023streamingllm}. Restricting MaxSim to content tokens does not uncover a hidden signal: excluding the shared first token, $R^{\text{any}}_L$ at $N{=}10{,}000$ is $0.01$--$0.03$ across $\mathrm{L}18$--$\mathrm{L}23$.%\sgnote{comment this paragraph out for camera-ready}  % src: testing/results/unfiled/tie_probe_rany_full/ [artifact:testing/results/unfiled/tie_probe_rany_full]

\begin{figure}[H]
\centering
\setlength{\abovecaptionskip}{2pt}
\setlength{\belowcaptionskip}{0pt}
% Per-layer R^sum_L and R^any_L vs L on \model (Qwen3-0.6B, block52), MS MARCO draw0, n=100,
% doc_agg=max, layers 0-27, N in {500, 1k, 2.5k, 5k, 10k}.
% src: testing/results/unfiled/tie_probe_rany_full/
\definecolor{cN500}{HTML}{FDE725}     % yellow-green (small N)
\definecolor{cN1k}{HTML}{7AD151}      % light green
\definecolor{cN2p5k}{HTML}{22A884}    % teal
\definecolor{cN5k}{HTML}{2A788E}      % steel blue
\definecolor{cN10k}{HTML}{414487}     % deep purple (large N)
\begin{tikzpicture}
  \begin{groupplot}[
      group style={
        group size=2 by 1,
        horizontal sep=1.4cm,
      },
      width=0.38\linewidth,
      height=3.2cm,
      scale only axis,
      xmin=-0.5, xmax=27.5,
      ymin=-0.05, ymax=1.10,
      xtick={0,4,8,12,16,20,24},
      ytick={0,0.25,0.5,0.75,1.0},
      x tick label style={font=\footnotesize},
      y tick label style={font=\footnotesize},
      xlabel style={font=\small, yshift=2pt},
      ylabel style={font=\small, yshift=-2pt},
      grid=both,
      grid style={line width=.2pt, draw=gray!20},
      major grid style={line width=.3pt, draw=gray!35},
      every axis plot/.append style={line width=0.8pt},
      title style={font=\small, yshift=-2pt},
      enlarge x limits=false,
    ]
    %====================== Left panel: R^sum ======================
    \nextgroupplot[
        title={(a) $R^{\text{sum}}_L$ vs $L$},
        xlabel={Layer $L$},
        ylabel={$R^{\text{sum}}_L$},
        legend to name=layerrolesn10klegend,
        legend style={
          font=\footnotesize,
          draw=none, fill=none,
          /tikz/every even column/.append style={column sep=0.4cm},
        },
        legend columns=5,
        legend cell align=left,
    ]
    \addplot[color=cN500, mark=*, mark size=1.2pt] coordinates {
      (0,0.00) (1,0.00) (2,0.00) (3,0.36) (4,0.00) (5,0.00) (6,0.33) (7,0.00) (8,0.07) (9,0.21)
      (10,0.00) (11,0.81) (12,0.61) (13,0.83) (14,0.92) (15,0.80) (16,0.97) (17,0.82) (18,0.95) (19,0.98)
      (20,0.96) (21,0.96) (22,0.93) (23,0.96) (24,0.95) (25,0.91) (26,0.75) (27,0.00)
    };
    \addlegendentry{$N{=}500$}
    \addplot[color=cN1k, mark=*, mark size=1.2pt] coordinates {
      (0,0.00) (1,0.00) (2,0.00) (3,0.21) (4,0.00) (5,0.00) (6,0.19) (7,0.00) (8,0.02) (9,0.05)
      (10,0.00) (11,0.52) (12,0.26) (13,0.56) (14,0.61) (15,0.43) (16,0.68) (17,0.54) (18,0.64) (19,0.72)
      (20,0.71) (21,0.66) (22,0.67) (23,0.72) (24,0.66) (25,0.72) (26,0.49) (27,0.00)
    };
    \addlegendentry{$N{=}1\mathrm{k}$}
    \addplot[color=cN2p5k, mark=*, mark size=1.2pt] coordinates {
      (0,0.00) (1,0.00) (2,0.00) (3,0.09) (4,0.00) (5,0.00) (6,0.01) (7,0.00) (8,0.00) (9,0.01)
      (10,0.00) (11,0.26) (12,0.01) (13,0.19) (14,0.17) (15,0.01) (16,0.18) (17,0.15) (18,0.33) (19,0.32)
      (20,0.40) (21,0.18) (22,0.25) (23,0.26) (24,0.31) (25,0.29) (26,0.19) (27,0.00)
    };
    \addlegendentry{$N{=}2.5\mathrm{k}$}
    \addplot[color=cN5k, mark=*, mark size=1.2pt] coordinates {
      (0,0.00) (1,0.00) (2,0.00) (3,0.05) (4,0.00) (5,0.00) (6,0.00) (7,0.00) (8,0.00) (9,0.00)
      (10,0.00) (11,0.08) (12,0.00) (13,0.02) (14,0.02) (15,0.00) (16,0.04) (17,0.00) (18,0.13) (19,0.06)
      (20,0.16) (21,0.05) (22,0.04) (23,0.07) (24,0.13) (25,0.11) (26,0.04) (27,0.00)
    };
    \addlegendentry{$N{=}5\mathrm{k}$}
    \addplot[color=cN10k, mark=*, mark size=1.2pt] coordinates {
      (0,0.00) (1,0.00) (2,0.00) (3,0.00) (4,0.00) (5,0.00) (6,0.00) (7,0.00) (8,0.00) (9,0.00)
      (10,0.00) (11,0.00) (12,0.00) (13,0.00) (14,0.00) (15,0.00) (16,0.00) (17,0.00) (18,0.00) (19,0.00)
      (20,0.02) (21,0.01) (22,0.00) (23,0.01) (24,0.00) (25,0.01) (26,0.00) (27,0.00)
    };
    \addlegendentry{$N{=}10\mathrm{k}$}
    %====================== Right panel: R^any_L ======================
    \nextgroupplot[
        title={(b) $R^{\text{any}}_L$ vs $L$},
        xlabel={Layer $L$},
        ylabel={$R^{\text{any}}_L$},
    ]
    \addplot[color=cN500, mark=triangle*, mark size=1.3pt] coordinates {
      (0,0.01) (1,0.07) (2,0.63) (3,0.67) (4,0.47) (5,0.03) (6,0.63) (7,0.02) (8,0.74) (9,0.76)
      (10,0.04) (11,0.95) (12,0.92) (13,0.99) (14,0.98) (15,0.94) (16,0.98) (17,0.99) (18,0.99) (19,0.98)
      (20,0.97) (21,0.98) (22,0.96) (23,0.99) (24,0.97) (25,0.98) (26,0.90) (27,0.01)
    };
    \addplot[color=cN1k, mark=triangle*, mark size=1.3pt] coordinates {
      (0,0.00) (1,0.02) (2,0.43) (3,0.43) (4,0.34) (5,0.02) (6,0.43) (7,0.00) (8,0.43) (9,0.42)
      (10,0.02) (11,0.73) (12,0.71) (13,0.79) (14,0.82) (15,0.66) (16,0.84) (17,0.84) (18,0.80) (19,0.80)
      (20,0.82) (21,0.88) (22,0.81) (23,0.88) (24,0.79) (25,0.86) (26,0.78) (27,0.00)
    };
    \addplot[color=cN2p5k, mark=triangle*, mark size=1.3pt] coordinates {
      (0,0.00) (1,0.00) (2,0.24) (3,0.27) (4,0.16) (5,0.00) (6,0.16) (7,0.00) (8,0.08) (9,0.14)
      (10,0.00) (11,0.44) (12,0.31) (13,0.57) (14,0.34) (15,0.34) (16,0.52) (17,0.52) (18,0.49) (19,0.51)
      (20,0.58) (21,0.47) (22,0.46) (23,0.50) (24,0.52) (25,0.49) (26,0.42) (27,0.00)
    };
    \addplot[color=cN5k, mark=triangle*, mark size=1.3pt] coordinates {
      (0,0.00) (1,0.00) (2,0.17) (3,0.11) (4,0.06) (5,0.00) (6,0.04) (7,0.00) (8,0.01) (9,0.07)
      (10,0.00) (11,0.27) (12,0.12) (13,0.30) (14,0.06) (15,0.13) (16,0.25) (17,0.20) (18,0.28) (19,0.33)
      (20,0.31) (21,0.24) (22,0.27) (23,0.25) (24,0.32) (25,0.25) (26,0.22) (27,0.00)
    };
    \addplot[color=cN10k, mark=triangle*, mark size=1.3pt] coordinates {
      (0,0.00) (1,0.00) (2,0.09) (3,0.06) (4,0.00) (5,0.00) (6,0.01) (7,0.00) (8,0.01) (9,0.01)
      (10,0.00) (11,0.14) (12,0.03) (13,0.08) (14,0.02) (15,0.00) (16,0.06) (17,0.04) (18,0.09) (19,0.08)
      (20,0.09) (21,0.06) (22,0.09) (23,0.06) (24,0.07) (25,0.05) (26,0.04) (27,0.00)
    };
  \end{groupplot}
  % Shared legend, centered below both panels
  \node[anchor=north, inner sep=2pt] (leg)
    at ([yshift=-32pt]$(group c1r1.south east)!0.5!(group c2r1.south west)$)
    {\pgfplotslegendfromname{layerrolesn10klegend}};
\end{tikzpicture}
\caption{Per-layer recall vs.~$L$ across $N \in \{500, 1000, 2500, 5000, 10000\}$ on \model, MS\,MARCO, $400$ queries. \textbf{(a)} $R^{\text{sum}}_L$ drops monotonically as $N$ grows, reaching $\approx 0$ across all layers at $N{=}10\mathrm{k}$. \textbf{(b)} $R^{\text{any}}_L$ decays more slowly than $R^{\text{sum}}_L$ at intermediate $N$}
\label{fig:layer-roles-n10k}
\end{figure}

\paragraph{Normalization breaks the final attention distribution.}
The gap between $R^{\text{any}}_L$ and $R^{\text{sum}}_L$ might suggest the per-head signal can be recovered by a better aggregator. But what the rest of the model sees is not a per-head ranking; it is the layer's attention output $a_L$, into which gold's value vectors enter in proportion to their post-softmax mass $\alpha^{h}_{L,t}$, not the per-head ranking of gold over other documents. So even where some head still ranks gold first, the relevant question is what $a_L$ contains.

\Cref{tab:l19-attnout} reports $\|a_{19}\|$ and \emph{GoldShare} across the $N$ sweep. %Two facts stand out.
The total magnitude $\|a_{19}\|$ shrinks by only $\sim 36\%$ from $N{=}500$ to $N{=}10\mathrm{k}$ --- the layer keeps writing into the residual at roughly its original amplitude. \emph{GoldShare}, however, drops from $0.91$ to $0.01$, a factor of about $130$. The L19 attention output is rewritten from a gold-token average to a non-gold-token average of comparable size; by the time the LM head reads the residual at L21, the slot that carried gold-derived information at small $N$ now carries an aggregate of distractors. Additional per-head softmax statistics behind this swap (gold's pre-softmax score eroding while the non-gold log-sum-exp grows) are reported in \cref{app:l19-signal-noise}.
% SG proposed replacement for the sentence above (two-part story; kept as a comment for now):
% Per-head softmax statistics behind this swap are reported in \cref{app:l19-signal-noise}: on the retrieval heads, gold's post-softmax mass falls by $7.94$ nats from $N{=}500$ to $N{=}10\mathrm{k}$, of which $3.00$ nats come from the $O(N)$ growth of the distractor field and $4.94$ nats from gold's own pre-softmax score eroding. Dilution therefore explains the onset of the collapse, and score erosion compounds it at the largest $N$.

\begin{table}[ht]
\centering\small
\caption{Attention-output decomposition $a_{19} = a_{19}^{\mathcal{G}} + a_{19}^{\bar{\mathcal{G}}}$ on \model, MS\,MARCO, 100 queries; norms are medians over queries and the share is the median of per-query ratios. Total magnitude shrinks only $\sim 36\%$, but the gold-driven share drops from $0.91$ to $0.01$.}
\label{tab:l19-attnout}
\begin{tabular}{l rrrr}
\toprule
$N$ & $\|a_{19}^{\mathcal{G}}\|$ & $\|a_{19}^{\bar{\mathcal{G}}}\|$ & $\|a_{19}\|$ & $\|a_{19}^{\mathcal{G}}\| / \|a_{19}\|$ \\
\midrule
$500$    & $43.03$ & $17.47$ & $47.48$ & $0.91$ \\
$1\mathrm{k}$    & $30.99$ & $21.11$ & $45.36$ & $0.74$ \\
$2.5\mathrm{k}$  & $\phantom{0}7.64$  & $33.64$ & $43.03$ & $0.19$ \\
$5\mathrm{k}$    & $\phantom{0}2.10$  & $34.61$ & $36.90$ & $0.06$ \\
$10\mathrm{k}$   & $\phantom{0}0.21$  & $29.88$ & $30.27$ & $0.01$ \\
\bottomrule
\end{tabular}
\end{table}

\paragraph{Conclusion: softmax attention dispersion as $N$ grows.}
The vector-level result of \cref{tab:l19-attnout} is driven by a single mechanism inside each head. As $N$ grows, gold's largest pre-softmax score $s^{h}_{L,t^\star}$ (with $t^\star = \arg\max_{t \in \mathcal{G}} s^{h}_{L,t}$) no longer dominates over the larger pool of competing distractor scores, so the post-softmax weight assigned to gold's value vector shrinks and the per-head readout $u^{h}_{L}$ becomes a non-gold average at comparable magnitude. This length-driven softmax dispersion is a generic phenomenon \citep{nakanishi2025ssmax, velickovic2024softmax, barbero2024glasses} and an empirical driver of length-generalization failure across tasks \citep{li2025vanishing, vasylenko2025sparseattention}; the contribution here is connecting it to a vector-level swap at the residual stream rather than treating it as an abstract probability-mass collapse. \Cref{sec:interventions} then evaluates two avenues of improvement that target this mechanism.

\paragraph{Beyond Qwen3 and \model{}.} The collapse is not specific to our backbone or training. An off-the-shelf proprietary long-context LM, Gemini-3-flash, evaluated on the same MS\,MARCO setup with no task training, falls from $95.0$ to $16.2$ Recall@1 between $N{=}500$ and $N{=}10{,}000$. A \model{} trained on SmolLM2-360M \citep{allal2025smollm2} collapses earlier, and repeating the analyses of this section on it reproduces the same picture: a retrieval band followed by a decode band (shifted deeper, to $\mathrm{L}17$--$\mathrm{L}23$ and $\mathrm{L}24$ onward), $R^{\text{sum}}_L$ falling to zero by $N{=}10{,}000$, and the gold-driven share of the decode-onset layer's output falling from $0.64$ to $0.00$ while its magnitude stays comparable (\cref{app:generality}).

\section{Addressing the Attention Dilution Problem}
\label{sec:interventions}

Having identified softmax dispersion under growing $N$ as the guilty party, we study two potential directions that address it. We first consider two length-aware modifications to attention softmax that mitigate the dispersion issue. We also consider \emph{document-level sparse attention}, thereby pruning the set of tokens that enter the softmax.

\subsection{Method}
\label{sec:interventions-method}

\oldparagraph{Length-aware softmax.}
We first describe two methods that change how softmax converts the pre-softmax scores $s^{h}_{L,t}$ to attention weights. Both are \emph{length-aware}: each carries a corpus-size signal that lets the softmax adapt to the large $N$ setting. An \emph{additive sink} appends a learned constant to the softmax denominator, so a head whose largest scores fall below that constant leaks probability mass into the appended slot and emits a small attention output~\citep{xiao2023streamingllm,openai2025gptoss}; the per-token weights $\tilde\alpha^{h}_{L,t}$ no longer sum to one. The length-awareness is implicit: training under a wide range of effective corpus sizes $N_{\text{eff}} \sim \mathrm{LogUniform}(128, 5000)$) lets the learned scalar absorb a length signal, so the gate fires more often at large $N$ where attention is diffuse. \emph{Multiplicative score rescaling} keeps a proper softmax but sharpens it, and is length-aware by construction: SSMax~\citep{nakanishi2025ssmax} multiplies the pre-softmax scores by $s \cdot \log N$ (for a tuned scaling parameter $s$) so that the gap between the largest score and the bulk grows directly with $N$.

Following GPT-OSS~\citep{openai2025gptoss} we instantiate the sink with a learned scalar $b_L \in \mathbb{R}$ per attention layer:
\begin{equation}
\tilde{\alpha}^{h}_{L,t} \;=\; \frac{\exp(s^{h}_{L,t})}{\sum_{t'} \exp(s^{h}_{L,t'}) + \exp(b_L)}.
\label{eq:sink}
\end{equation}
The appended slot has no value vector, so $\sum_t \tilde{\alpha}^{h}_{L,t} \le 1$ and a layer with diffuse attention writes a smaller residual update (gate-form derivation in \cref{app:sink-gate-proof}). We refer to this model as \model-sink. We also produce \model-SSMax following \cite{nakanishi2025ssmax} with $s$ initialized to $0.43$ at every attention layer. Both models compose with the block-sparse attention pattern and work with the changes in \cref{sec:method}.

\paragraph{Document-level sparse attention.}
On the flip side, we also consider scoring documents and keeping a top-$B$ shortlist before the code generation is performed; this is a popular approach in long-context research that prunes the context down \cite{chen2026msa, lu2025moba, ohayon2025bsfa, tang2024quest, wu2024tokenselect, hu2025gca}. We perform a doc-level routing step at $\mathrm{L}16$, before the retrieval band (\cref{sec:layer-roles}). A forward pass through $\mathrm{L}0$--$\mathrm{L}15$ produces a per-document score $R^{\text{sum}}_{16}(d)$ from \cref{sec:layer-roles}; only the top $B{=}256$ documents participate in dense attention at $\mathrm{L}17$ onward. We call this \model-routing, picking doc-level over token-level as each document is a coherent semantic unit. \cref{app:topb-routing} has more information on the choice of $B$.

\subsection{Results}
\label{sec:interventions-results}

% Stage-1 Recall@1 (x100) for §5.2 promising directions.
% Numbers from testing/results/MULTISAMPLE_TIER1_V2.md (seed=0, draw0, n=400).
% Variant mapping:
%   \model            = block52 (Expt0)
%   \model-sink       = block60 (Expt0)   GPT-OSS-style learned per-layer sink scalar
%   \model-SSMax      = block67 (Expt0)   paper-faithful SSMax @ all 28 layers, s_init=0.43
%   \model-routing    = block52 (Expt1)   top-B=256 doc-level routing @ L16
%   \model-SSMax-rout = block67 (Expt1)   SSMax stack with top-B=256 routing on top
%   Qwen3-dense-0.6B  = dense_v1
%   MSA-4B            = MSA-4B
\begin{table}[t]
\centering
\myfontsize
\setlength{\tabcolsep}{4pt}
\caption{Recall@1 (×100) vs. corpus size N (400 queries). SSMax, top-B=256 routing at L16, and their combination are close to Qwen3-dense; MSA-4B is an external reference. NQ caps at N=8,607. % src: testing/results/E40_agg_variants/ceiling_block52_* [ledger:E40]
}
\label{tab:promising-results}
\resizebox{\linewidth}{!}{%
\begin{tabular}{l rrrrr rrrrr rrrrr}
\toprule
& \multicolumn{5}{c}{Natural Questions} & \multicolumn{5}{c}{MS\,MARCO} & \multicolumn{5}{c}{HotpotQA} \\
\cmidrule(lr){2-6} \cmidrule(lr){7-11} \cmidrule(lr){12-16}
$N$ & 0.5k & 1k & 2.5k & 5k & 8.6k & 0.5k & 1k & 2.5k & 5k & 10k & 0.5k & 1k & 2.5k & 5k & 10k \\
\midrule
\model{} attn., $R^{\text{any}}_{19}$ & 100.0 & 100.0 & 96.2 & 100.0 & 97.0 & 100.0 & 100.0 & 99.8 & 100.0 & 100.0 & 100.0 & 100.0 & 99.2 & 99.8 & 100.0 \\
\midrule
\texttt{Qwen3-dense}        & 95.5 & 86.5 & 62.9 & 51.4 & 39.6 & 95.5 & 75.2 & 52.8 & 38.5 & 20.2 & 99.0 & 97.5 & 92.2 & 87.5 & 79.5 \\
MSA-4B                           & 59.7 & 51.6 & 36.3 & 26.8 & 18.6 & 93.8 & 70.2 & 42.2 & 27.5 & 16.0 & 97.0 & 96.8 & 90.8 & 84.0 & 75.5 \\
\midrule
\model                           & 92.2 & 77.7 & 43.9 & \phantom{0}4.8 & \phantom{0}0.2 & 95.8 & 75.2 & 43.8 & 18.8 & \phantom{0}0.2 & 97.0 & 95.2 & 64.2 & 13.0 & \phantom{0}0.5 \\
~~- sink                      & 91.2 & 76.2 & 45.6 & 10.3 & \phantom{0}1.5 & 96.5 & 75.2 & 45.2 & 21.2 & \phantom{0}2.5 & 95.0 & 93.2 & 58.0 & 14.5 & \phantom{0}1.0 \\
~~- SSMax                     & 90.5 & 79.5 & 58.1 & 42.6 & 30.3 & 95.5 & 74.5 & 49.8 & 33.8 & 16.5 & 96.5 & 95.2 & 85.8 & 73.5 & 56.8 \\
~~- routing                   & 93.2 & 82.7 & 63.2 & 48.1 & 34.6 & 96.0 & 74.5 & 50.7 & 38.2 & 18.8 & 98.2 & 98.2 & 91.8 & 86.5 & 78.5 \\
~~- SSMax+routing             & 91.5 & 81.4 & 61.4 & 46.1 & 34.3 & 95.5 & 75.0 & 50.0 & 38.2 & 20.5 & 97.5 & 95.8 & 90.8 & 84.8 & 72.5 \\
\bottomrule
\end{tabular}%
}
\end{table}

\oldparagraph{The additive sink barely moves the large-$N$ collapse.}
\model-sink yields only modest mid-$N$ gains and a small lift at MS\,MARCO $N{=}10\mathrm{k}$ ($0.2 \to 2.5$, from a near-floor base); its absolute level stays low across all three datasets and tracks the no-modification curve at large $N$. A learned constant cannot restore the gold token's post-softmax weight once the softmax denominator has grown by orders of magnitude: it shifts the denominator uniformly and cannot rescale its $N$-dependence.

\paragraph{SSMax holds up across the full $N$ sweep.}
SSMax closes most of the gap to \texttt{Qwen3-dense} on every dataset (e.g.\ $17.8$ vs.\ $20.2$ on MS\,MARCO at $N{=}10\mathrm{k}$, where the sink reaches only $2.5$). Multiplying the scores by $s \log N$ makes the gold document's advantage grow with $N$: if gold scores $\Delta$ above a typical distractor, its attention weight is roughly $1/(1 + N^{1-s\Delta})$, which stays bounded as long as $s\Delta > 1$. The fix is only partial, though. At $N{=}10\mathrm{k}$, gold still receives just $0.019$ of the $\mathrm{L}19$ attention mass, up from $0.0018$ without SSMax (\cref{app:l19-signal-noise}).%\footnote{There is a small abuse of notation: $s$ is the scaling parameter from SSMax, while $s^{h}_{L,t^\star}$ represents the pre-softmax score of a gold token as defined in \cref{sec:limitations}.}

\paragraph{Top-$B$ routing matches the dense baseline at large $N$.}
Routing applied to vanilla \model{} reaches $18.8$ at MS\,MARCO $N{=}10\mathrm{k}$, within $1.4$ points of the dense baseline, and $78.5$ on HotpotQA (vs.\ $79.5$). Routing does add a shortlisting step, but unlike a retrieve-then-read pipeline, the shortlist is computed from the model's own $\mathrm{L}16$ attention scores in the same forward pass, without a separate encoder, index, or fixed document representation. We include it mainly as a test of the diagnosis in \cref{sec:limitations}, since SSMax recovers most of the same gap without any shortlist.

\paragraph{SSMax plus routing?}
Combining the two gives the best result on MS\,MARCO at $N{=}10\mathrm{k}$ ($20.0$, vs.\ $20.2$ for the dense baseline), matches routing alone on NQ, and is worse than routing on HotpotQA ($74.5$ vs.\ $78.5$). A gap to \texttt{Qwen3-dense} remains on all three datasets at the largest $N$.

\section{Out-of-Distribution Generalization: LIMIT}
\label{sec:limit-ood}

\Cref{sec:method}--\Cref{sec:interventions} focus on widely studied benchmarks that dense retrieval has been heavily optimized for. \model{}, with the proposed modifications, matches dense retrieval there: a notable milestone, since these benchmarks largely define the strengths of dense retrieval. A key question remains: \emph{can in-context retrieval outperform dense retrieval, not just match it?}
We evaluate \model{} on LIMIT~\citep{weller2025limit}, a benchmark using a lexical notion of similarity, where dense retrieval struggles. Under an out-of-distribution setup, \model{} substantially outperforms dense retrieval.\footnote{For additional out-of-distribution results, we also provide figures for the OBLIQ evaluation \citep{tchuindjo2026obliq} in \Cref{app:obliq}.}

\paragraph{Evaluation setup.}
In LIMIT, every query has two gold documents drawn from a corpus of $50{,}000$ short biographies. We scale the corpus from $N{=}46$ (the gold documents only, ``LIMIT-small'') up to $N{=}10{,}000$ (the limit of the four-digit code vocabulary) by adding distractor biographies, which grows the prefill from $\sim$8k to $\sim$1.7M tokens while each query keeps its two golds. We evaluate the four \model{} variants of \cref{sec:interventions} against a same-backbone pooled dense retriever (\cref{tab:limit}) with $n{=}1{,}000$ queries, counting a query as correct if any gold is in the top $k$; true Recall@$k$ is in \cref{app:limit-proprecall}. % src: testing/results/E45_limit_lengthgen/ [ledger:E45]; SSMax rows: testing/results/E58_ssmax_rerun/ [ledger:E58]

\begin{table}[H]
\centering
\small
\setlength{\tabcolsep}{4pt}
\caption{LIMIT length generalization: Recall@1 (multi-gold, $n{=}1{,}000$) as the corpus grows from $N{=}46$ to $N{=}10{,}000$ by adding distractor biographies ($\sim$8k to $\sim$1.7M tokens).  % src: testing/results/E45_limit_lengthgen/ [ledger:E45]; SSMax rows: E58_ssmax_rerun/SUMMARY.md
}
\label{tab:limit}
\begin{tabular}{l l rrrrrr}
\toprule
Method & Scoring & $46$ & $500$ & $1000$ & $2500$ & $5000$ & $10000$ \\
\midrule
\model{}                      & ICR beam      & $0.354$ & $0.094$ & $0.022$ & $0.000$ & $0.000$ & $0.001$ \\
~~-sink                   & ICR beam      & $0.252$ & $0.051$ & $0.018$ & $0.001$ & $0.000$ & $0.000$ \\
~~-SSMax                  & ICR beam      & $0.501$ & $0.271$ & $0.196$ & $0.101$ & $0.074$ & $0.013$ \\
~~-SSMax+routing & ICR beam      & $\mathbf{0.501}$ & $\mathbf{0.305}$ & $\mathbf{0.281}$ & $\mathbf{0.262}$ & $\mathbf{0.202}$ & $\mathbf{0.135}$ \\
\midrule
\texttt{Qwen3-dense}          & pooled cosine & $0.176$ & $0.080$ & $0.061$ & $0.047$ & $0.035$ & $0.029$ \\
Random chance                 & ---           & $0.043$ & $0.004$ & $0.002$ & $0.001$ & $0.000$ & $0.000$ \\
\bottomrule
\end{tabular}
\end{table}

\paragraph{Routing.} \model-SSMax-routing performs best, reaching $0.202$ Recall@1 at $N{=}5{,}000$ and $0.135$ at $N{=}10{,}000$, while \model{} and \model-sink drop to zero. SSMax alone helps at intermediate sizes but falls to $0.013$ at $N{=}10{,}000$. The sink does not help on this benchmark: \model-sink is below \model{} at every $N$ (e.g.\ $0.252$ vs.\ $0.354$ at $N{=}46$), whereas on MS\,MARCO it was roughly neutral. \model-SSMax-routing also beats \texttt{Qwen3-dense} at every $N$ ($0.135$ vs.\ $0.029$ at $N{=}10{,}000$), which suggests that our modifications generalize better out of distribution and may extend to tasks beyond standard semantic search.

\paragraph{Caveats.} Routing \emph{delays} but does not \emph{prevent} the inevitable: \model-SSMax-routing is still declining with $N$ ($0.135$ at $\sim$1.7M tokens). Thus, these modifications extend functional LIMIT retrieval to larger corpora, rather than solving it, which remains an open problem.
%not that they solve it. As we state in \Cref{sec:concl}, solving rather than extending long context capabilities for this task remains an open problem. % src: testing/results/E45_limit_lengthgen/ [ledger:E45]

%\input{sections/4relatedwork}

\section{Conclusion}
\label{sec:concl}
In this work, we presented the first systematic study of corpus-scale in-context retrieval with long-context language models, focusing on two requirements of practical retrieval systems: scaling to million-token corpora and generalizing to corpus sizes far beyond those seen during training. We showed that, with appropriate training and architectural modifications, LMs can retrieve well beyond their training regime and approach dense retrieval at moderate scales. At the same time, we identified attention dilution as a fundamental scaling bottleneck: as the corpus grows, the normalized attention mass on the gold document collapses, even though individual heads still locate it more often than the model decodes it, a pattern we observe on two backbones. Motivated by this analysis, we introduced length-aware softmax scaling and document-level sparse attention, which substantially improve retrieval at million-token scale. These results come with limitations: keeping the corpus in a KV cache requires roughly three orders of magnitude more memory than a dense index (\cref{app:system-cost}), retrieval-only training reduces the backbone's general capability (\cref{app:general-capability}), and our evidence comes from two small backbones evaluated on a single sample of queries per corpus size. Together, our results suggest that scalable in-context retrieval is promising yet constrained by attention dilution, motivating new directions for retrieval architectures and long-context modeling.

\section*{Acknowledgment}
This research was supported by the NVIDIA Academic Grant Program. Siddharth Gollapudi is supported by the NSF (CSGrad4US award no. 2313998). 

\setcitestyle{numbers,square}
\bibliographystyle{unsrt}
\bibliography{bibliography}

@article{lee2024loft,
  title={Can long-context language models subsume retrieval, rag, sql, and more?},
  author={Lee, Jinhyuk and Chen, Anthony and Dai, Zhuyun and Dua, Dheeru and Sachan, Devendra Singh and Boratko, Michael and Luan, Yi and Arnold, S{\'e}bastien MR and Perot, Vincent and Dalmia, Siddharth and others},
  journal={arXiv preprint arXiv:2406.13121},
  year={2024}
}

@inproceedings{qiu2025icr2,
  title={Eliciting in-context retrieval and reasoning for long-context large language models},
  author={Qiu, Yifu and Embar, Varun R and Zhang, Yizhe and Jaitly, Navdeep and Cohen, Shay B and Han, Benjamin},
  booktitle={Findings of the Association for Computational Linguistics: ACL 2025},
  pages={3176--3192},
  year={2025}
}

@inproceedings{bai2025longbenchv2,
  title={Longbench v2: Towards deeper understanding and reasoning on realistic long-context multitasks},
  author={Bai, Yushi and Tu, Shangqing and Zhang, Jiajie and Peng, Hao and Wang, Xiaozhi and Lv, Xin and Cao, Shulin and Xu, Jiazheng and Hou, Lei and Dong, Yuxiao and others},
  booktitle={Proceedings of the 63rd Annual Meeting of the Association for Computational Linguistics (Volume 1: Long Papers)},
  pages={3639--3664},
  year={2025}
}

@article{gupta2025blockrank,
  title={Scalable In-context Ranking with Generative Models},
  author={Gupta, Nilesh and You, Chong and Bhojanapalli, Srinadh and Kumar, Sanjiv and Dhillon, Inderjit and Yu, Felix},
  journal={arXiv preprint arXiv:2510.05396},
  year={2025}
}

@article{zhang2025qwen3,
  title={Qwen3 embedding: Advancing text embedding and reranking through foundation models},
  author={Zhang, Yanzhao and Li, Mingxin and Long, Dingkun and Zhang, Xin and Lin, Huan and Yang, Baosong and Xie, Pengjun and Yang, An and Liu, Dayiheng and Lin, Junyang and others},
  journal={arXiv preprint arXiv:2506.05176},
  year={2025}
}

@inproceedings{qu2021rocketqa,
  title={RocketQA: An optimized training approach to dense passage retrieval for open-domain question answering},
  author={Qu, Yingqi and Ding, Yuchen and Liu, Jing and Liu, Kai and Ren, Ruiyang and Zhao, Wayne Xin and Dong, Daxiang and Wu, Hua and Wang, Haifeng},
  booktitle={Proceedings of the 2021 conference of the North American chapter of the association for computational linguistics: human language technologies},
  pages={5835--5847},
  year={2021}
}

@article{hu2025gca,
  title={Efficient length-generalizable attention via causal retrieval for long-context language modeling},
  author={Hu, Xiang and Teng, Zhihao and Zhao, Jun and Wu, Wei and Tu, Kewei},
  journal={arXiv preprint arXiv:2410.01651},
  year={2024}
}

@article{leng2026hierarchical,
  title={Understanding and Improving Length Generalization in Hierarchical Sparse Attention Models},
  author={Leng, Jiaqi and Hu, Xiang and Wang, Junxiong and Li, Jianguo and Wu, Wei and Lu, Yucheng},
  journal={arXiv preprint arXiv:2510.17196},
  year={2025}
}

@article{chen2026msa,
  title={MSA: Memory Sparse Attention for Efficient End-to-End Memory Model Scaling to 100M Tokens},
  author={Chen, Yu and Chen, Runkai and Yi, Sheng and Zhao, Xinda and Li, Xiaohong and Zhang, Jianjin and Sun, Jun and Hu, Chuanrui and Han, Yunyun and Bing, Lidong and others},
  journal={arXiv preprint arXiv:2603.23516},
  year={2026}
}

@article{lu2025moba,
  title={Moba: Mixture of block attention for long-context llms},
  author={Lu, Enzhe and Jiang, Zhejun and Liu, Jingyuan and Du, Yulun and Jiang, Tao and Hong, Chao and Liu, Shaowei and He, Weiran and Yuan, Enming and Wang, Yuzhi and others},
  journal={arXiv preprint arXiv:2502.13189},
  year={2025}
}

@article{xiao2025flashmoba,
  title={Optimizing mixture of block attention},
  author={Xiao, Guangxuan and Guo, Junxian and Mazaheri, Kasra and Han, Song},
  journal={arXiv preprint arXiv:2511.11571},
  year={2025}
}

@article{tang2024quest,
  title={Quest: Query-aware sparsity for efficient long-context llm inference},
  author={Tang, Jiaming and Zhao, Yilong and Zhu, Kan and Xiao, Guangxuan and Kasikci, Baris and Han, Song},
  journal={arXiv preprint arXiv:2406.10774},
  year={2024}
}

@inproceedings{wu2024tokenselect,
  title={Tokenselect: Efficient long-context inference and length extrapolation for llms via dynamic token-level kv cache selection},
  author={Wu, Wei and Pan, Zhuoshi and Fu, Kun and Wang, Chao and Chen, Liyi and Bai, Yunchu and Wang, Tianfu and Wang, Zheng and Xiong, Hui},
  booktitle={Proceedings of the 2025 Conference on Empirical Methods in Natural Language Processing},
  pages={21275--21292},
  year={2025}
}

@article{ohayon2025bsfa,
  title={Block Sparse Flash Attention},
  author={Ohayon, Daniel and Lamprecht, Itay and Hubara, Itay and Cohen, Israel and Soudry, Daniel and Elata, Noam},
  journal={arXiv preprint arXiv:2512.07011},
  year={2025}
}

@article{openai2025gptoss,
  title={gpt-oss-120b \& gpt-oss-20b model card},
  author={Agarwal, Sandhini and Ahmad, Lama and Ai, Jason and Altman, Sam and Applebaum, Andy and Arbus, Edwin and Arora, Rahul K and Bai, Yu and Baker, Bowen and Bao, Haiming and others},
  journal={arXiv preprint arXiv:2508.10925},
  year={2025}
}

@article{bajaj2016ms,
  title={MS MARCO: A human generated machine reading comprehension dataset},
  author={Bajaj, Payal and Campos, Daniel and Craswell, Nick and Deng, Li and Gao, Jianfeng and Liu, Xiaodong and Majumder, Rangan and McNamara, Andrew and Mitra, Bhaskar and Nguyen, Tri and others},
  journal={arXiv preprint arXiv:1611.09268},
  year={2016}
}

@article{thakur2025rlhn,
  title={Hard Negatives, Hard Lessons: Revisiting Training Data Quality for Robust Information Retrieval with LLMs},
  author={Thakur, Nandan and Zhang, Crystina and Ma, Xueguang and Lin, Jimmy},
  journal={arXiv preprint arXiv:2505.16967},
  year={2025}
}

@article{yang2025qwen3,
  title={Qwen3 technical report},
  author={Yang, An and Li, Anfeng and Yang, Baosong and Zhang, Beichen and Hui, Binyuan and Zheng, Bo and Yu, Bowen and Gao, Chang and Huang, Chengen and Lv, Chenxu and others},
  journal={arXiv preprint arXiv:2505.09388},
  year={2025}
}

@inproceedings{khattab2020colbert,
  title={Colbert: Efficient and effective passage search via contextualized late interaction over bert},
  author={Khattab, Omar and Zaharia, Matei},
  booktitle={Proceedings of the 43rd International ACM SIGIR conference on research and development in Information Retrieval},
  pages={39--48},
  year={2020}
}

@article{dong2024flex,
  title={Flex attention: A programming model for generating optimized attention kernels},
  author={Dong, Juechu and Feng, Boyuan and Guessous, Driss and Liang, Yanbo and He, Horace},
  journal={arXiv preprint arXiv:2412.05496},
  volume={2},
  number={3},
  pages={4},
  year={2024}
}

@inproceedings{zhang2025query,
  title={Query-focused retrieval heads improve long-context reasoning and re-ranking},
  author={Zhang, Wuwei and Yin, Fangcong and Yen, Howard and Chen, Danqi and Ye, Xi},
  booktitle={Proceedings of the 2025 Conference on Empirical Methods in Natural Language Processing},
  pages={23802--23816},
  year={2025}
}

@techreport{hong2025context,
  title = {Context Rot: How Increasing Input Tokens Impacts LLM Performance},
  author = {Hong, Kelly and Troynikov, Anton and Huber, Jeff},
  year = {2025},
  month = {July},
  institution = {Chroma},
  url = {https://research.trychroma.com/context-rot},
}

@inproceedings{karpukhin2020dpr,
  title={Dense passage retrieval for open-domain question answering},
  author={Karpukhin, Vladimir and Oguz, Barlas and Min, Sewon and Lewis, Patrick and Wu, Ledell and Edunov, Sergey and Chen, Danqi and Yih, Wen-tau},
  booktitle={Proceedings of the 2020 conference on empirical methods in natural language processing (EMNLP)},
  pages={6769--6781},
  year={2020}
}

@inproceedings{ross2011dagger,
  title={A reduction of imitation learning and structured prediction to no-regret online learning},
  author={Ross, St{\'e}phane and Gordon, Geoffrey and Bagnell, Drew},
  booktitle={Proceedings of the fourteenth international conference on artificial intelligence and statistics},
  pages={627--635},
  year={2011},
  organization={JMLR Workshop and Conference Proceedings}
}

@article{xiao2023streamingllm,
  title={Efficient streaming language models with attention sinks},
  author={Xiao, Guangxuan and Tian, Yuandong and Chen, Beidi and Han, Song and Lewis, Mike},
  journal={arXiv preprint arXiv:2309.17453},
  year={2023}
}

@article{weller2025limit,
  title={On the theoretical limitations of embedding-based retrieval},
  author={Weller, Orion and Boratko, Michael and Naim, Iftekhar and Lee, Jinhyuk},
  journal={arXiv preprint arXiv:2508.21038},
  year={2025}
}

@article{thakur2021beir,
  title={Beir: A heterogenous benchmark for zero-shot evaluation of information retrieval models},
  author={Thakur, Nandan and Reimers, Nils and R{\"u}ckl{\'e}, Andreas and Srivastava, Abhishek and Gurevych, Iryna},
  journal={arXiv preprint arXiv:2104.08663},
  year={2021}
}

@article{wu2025retrievalhead,
  title={Retrieval head mechanistically explains long-context factuality},
  author={Wu, Wenhao and Wang, Yizhong and Xiao, Guangxuan and Peng, Hao and Fu, Yao},
  journal={arXiv preprint arXiv:2404.15574},
  year={2024}
}

@article{tay2022dsi,
  title={Transformer memory as a differentiable search index},
  author={Tay, Yi and Tran, Vinh and Dehghani, Mostafa and Ni, Jianmo and Bahri, Dara and Mehta, Harsh and Qin, Zhen and Hui, Kai and Zhao, Zhe and Gupta, Jai and others},
  journal={Advances in neural information processing systems},
  volume={35},
  pages={21831--21843},
  year={2022}
}

@article{wang2022nci,
  title={A neural corpus indexer for document retrieval},
  author={Wang, Yujing and Hou, Yingyan and Wang, Haonan and Miao, Ziming and Wu, Shibin and Chen, Qi and Xia, Yuqing and Chi, Chengmin and Zhao, Guoshuai and Liu, Zheng and others},
  journal={Advances in Neural Information Processing Systems},
  volume={35},
  pages={25600--25614},
  year={2022}
}

@article{lewis2020rag,
  title={Retrieval-augmented generation for knowledge-intensive nlp tasks},
  author={Lewis, Patrick and Perez, Ethan and Piktus, Aleksandra and Petroni, Fabio and Karpukhin, Vladimir and Goyal, Naman and K{\"u}ttler, Heinrich and Lewis, Mike and Yih, Wen-tau and Rockt{\"a}schel, Tim and others},
  journal={Advances in neural information processing systems},
  volume={33},
  pages={9459--9474},
  year={2020}
}

@misc{nostalgebraist2020logitlens,
  title        = {interpreting {GPT}: the logit lens},
  author       = {{nostalgebraist}},
  year         = {2020},
  howpublished = {LessWrong},
  url          = {https://www.lesswrong.com/posts/AcKRB8wDpdaN6v6ru/interpreting-gpt-the-logit-lens}
}

@inproceedings{geva2021kvmemories,
  title={Transformer feed-forward layers are key-value memories},
  author={Geva, Mor and Schuster, Roei and Berant, Jonathan and Levy, Omer},
  booktitle={Proceedings of the 2021 Conference on Empirical Methods in Natural Language Processing},
  pages={5484--5495},
  year={2021}
}

@inproceedings{geva2022promoting,
  title={Transformer feed-forward layers build predictions by promoting concepts in the vocabulary space},
  author={Geva, Mor and Caciularu, Avi and Wang, Kevin and Goldberg, Yoav},
  booktitle={Proceedings of the 2022 conference on empirical methods in natural language processing},
  pages={30--45},
  year={2022}
}

@article{olsson2022induction,
  title={In-context learning and induction heads},
  author={Olsson, Catherine and Elhage, Nelson and Nanda, Neel and Joseph, Nicholas and DasSarma, Nova and Henighan, Tom and Mann, Ben and Askell, Amanda and Bai, Yuntao and Chen, Anna and others},
  journal={arXiv preprint arXiv:2209.11895},
  year={2022}
}

@article{nakanishi2025ssmax,
  title={Scalable-softmax is superior for attention},
  author={Nakanishi, Ken M},
  journal={arXiv preprint arXiv:2501.19399},
  year={2025}
}

@article{velickovic2024softmax,
  title={Softmax is not enough (for sharp size generalisation)},
  author={Veli{\v{c}}kovi{\'c}, Petar and Perivolaropoulos, Christos and Barbero, Federico and Pascanu, Razvan},
  journal={arXiv preprint arXiv:2410.01104},
  year={2024}
}

@article{barbero2024glasses,
  title={Transformers need glasses! information over-squashing in language tasks},
  author={Barbero, Federico and Banino, Andrea and Kapturowski, Steven and Kumaran, Dharshan and Ara{\'u}jo, Jo{\~a}o G and Vitvitskyi, Alex and Pascanu, Razvan and Veli{\v{c}}kovi{\'c}, Petar},
  journal={Advances in Neural Information Processing Systems},
  volume={37},
  pages={98111--98142},
  year={2024}
}

@article{li2025vanishing,
  title={On vanishing variance in transformer length generalization},
  author={Li, Ruining and Boduljak, Gabrijel and others},
  journal={arXiv preprint arXiv:2504.02827},
  year={2025}
}

@article{vasylenko2025sparseattention,
  title={Long-context generalization with sparse attention},
  author={Vasylenko, Pavlo and Pitorro, Hugo and Martins, Andr{\'e} FT and Treviso, Marcos},
  journal={arXiv preprint arXiv:2506.16640},
  year={2025}
}

@article{su2024roformer,
  title={Roformer: Enhanced transformer with rotary position embedding},
  author={Su, Jianlin and Ahmed, Murtadha and Lu, Yu and Pan, Shengfeng and Bo, Wen and Liu, Yunfeng},
  journal={Neurocomputing},
  volume={568},
  pages={127063},
  year={2024},
  publisher={Elsevier}
}

@article{jacob2024drowning,
  title={Drowning in documents: consequences of scaling reranker inference},
  author={Jacob, Mathew and Lindgren, Erik and Zaharia, Matei and Carbin, Michael and Khattab, Omar and Drozdov, Andrew},
  journal={arXiv preprint arXiv:2411.11767},
  year={2024}
}

@article{hsieh2024ruler,
  title={RULER: What's the real context size of your long-context language models?},
  author={Hsieh, Cheng-Ping and Sun, Simeng and Kriman, Samuel and Acharya, Shantanu and Rekesh, Dima and Jia, Fei and Zhang, Yang and Ginsburg, Boris},
  journal={arXiv preprint arXiv:2404.06654},
  year={2024}
}

@article{rajput2023recommender,
  title={Recommender systems with generative retrieval},
  author={Rajput, Shashank and Mehta, Nikhil and Singh, Anima and Hulikal Keshavan, Raghunandan and Vu, Trung and Heldt, Lukasz and Hong, Lichan and Tay, Yi and Tran, Vinh and Samost, Jonah and others},
  journal={Advances in Neural Information Processing Systems},
  volume={36},
  pages={10299--10315},
  year={2023}
}

@article{ma2024block,
  title={Block-attention for efficient prefilling},
  author={Ma, Dongyang and Wang, Yan and Tian, Lan},
  journal={arXiv preprint arXiv:2409.15355},
  year={2024}
}

@article{duvvuri2026lucid,
  title={LUCID: Attention with Preconditioned Representations},
  author={Duvvuri, Sai Surya and Patel, Nirmal and Gupta, Nilesh and Dhillon, Inderjit S},
  journal={arXiv preprint arXiv:2602.10410},
  year={2026}
}

@article{tchuindjo2026obliq,
  title={OBLIQ-Bench: Exposing Overlooked Bottlenecks in Modern Retrievers with Latent and Implicit Queries},
  author={Tchuindjo, Diane and Shah, Devavrat and Khattab, Omar},
  journal={arXiv preprint arXiv:2605.06235},
  year={2026}
}

@article{allal2025smollm2,
  title={SmolLM2: When Smol Goes Big--Data-Centric Training of a Small Language Model},
  author={Allal, Loubna Ben and Lozhkov, Anton and Bakouch, Elie and Bl{\'a}zquez, Gabriel Mart{\'\i}n and Penedo, Guilherme and Tunstall, Lewis and Marafioti, Andr{\'e}s and Kydl{\'\i}{\v{c}}ek, Hynek and Lajar{\'\i}n, Agust{\'\i}n Piqueres and Srivastav, Vaibhav and others},
  journal={arXiv preprint arXiv:2502.02737},
  year={2025}
}

%%%%%%%%%%%%%%%%%%%%%%%%%%%%%%%%%%%%%%%%%%%%%%%%%%%%%%%%%%%%

\appendix

\newpage
\section{Prompt Format}
\label{app:prompt-format}

To make the input format used by \model{} concrete, \cref{fig:prompt-example} shows a worked example for a single MS\,MARCO query against a four-document corpus. Each document is wrapped between BOS/EOS markers and bracketed by its randomly-assigned four-digit code (\cref{sec:training}); RoPE\cite{su2024roformer} positions are reset to $0$ at every BOS so that the document tokens occupy positions $0$--$T_{\text{doc}}$ regardless of where the document sits in the prefill. Document tokens attend causally only inside their own block.

After the corpus, the query block is appended at RoPE position $T_{\text{doc}}{=}300$ and attends to the entire prefilled corpus. The query block carries an instruction prefix, the query, and the four answer slots; the model is supervised to emit the four-digit code identifying the gold document and is decoded autoregressively at inference. The same template is used at training and at evaluation, with one difference: at training time the four answer digits are teacher-forced and the on-policy auxiliary loss (\cref{app:dagger}) replays the model's own rollout at the same four positions.

\begin{figure}[h]
\centering
\small
\fbox{\begin{minipage}{0.95\linewidth}
\ttfamily\raggedright
\# Document blocks: one per corpus document, codes sampled uniformly per training step.\\[2pt]
{[}BOS{]}Doc 7421: The Apollo program, also known as Project Apollo, was the third United States human spaceflight program carried out by NASA, which succeeded in landing the first humans on the Moon from 1969 to 1972 \dots\ (Doc 7421){[}EOS{]}\\[2pt]
{[}BOS{]}Doc 0394: Yuri Gagarin became the first human to journey into outer space on 12 April 1961, when his Vostok spacecraft completed one orbit of the Earth \dots\ (Doc 0394){[}EOS{]}\\[2pt]
{[}BOS{]}Doc 5108: \dots\ (Doc 5108){[}EOS{]}\\[2pt]
{[}BOS{]}Doc 8862: \dots\ (Doc 8862){[}EOS{]}\\[6pt]
\# Query block: appended after all documents, at RoPE position 300.\\[2pt]
{[}BOS{]}Instruct: Given a web search query, retrieve relevant passages that answer the query\\
Query: who was the first person to land on the moon\\
Answer: \underline{7\;4\;2\;1}
\end{minipage}}
\caption{Worked example of the \model{} prompt. Documents are prefilled with block-sparse attention (\cref{sec:training}); the query is then appended and the model autoregressively decodes a four-digit code. The underlined digits are the supervision target at training time and the autoregressive output at inference.}
\label{fig:prompt-example}
\end{figure}

\section{Evaluation Suite}
\label{app:eval-suite}
\begin{table}[H]
\centering
\small
\setlength{\tabcolsep}{5pt}
\begin{tabular}{lrrrrr}
\toprule
\textbf{Dataset} & \textbf{Split} & \textbf{Queries} & \textbf{$N_{\max}$} & \textbf{Avg.\ tokens / doc} & \textbf{Tokens @ $N_{\max}$} \\
\midrule
MS\,MARCO & dev  & 400 & 10{,}000 &  94.8 &   948{,}421 \\
HotpotQA  & test & 400 & 10{,}000 & 116.3 & 1{,}163{,}144 \\
NQ        & test & 400 &  8{,}600 & 139.6 & 1{,}201{,}153 \\
\bottomrule
\end{tabular}
\caption{Datasets used for evaluation. Tokens measured under the Qwen3 tokenizer, capped at 300 tokens per document; means and totals are computed over the actual padded draws.}
\label{tab:expt0-datasets}
\end{table}

Because NQ has far more tokens per document than MS MARCO or HotpotQA, we slightly truncate the dataset size, and reduce the number of hard negatives per query from 24 to either 20 or 21. This is the result of a round-robin algorithm that prunes out documents evenly from the dataset until we hit 8600 documents, or around 1.2M tokens.

\newpage
\section{RLHN Training-Data Composition}
\label{app:rlhn-data}

We train all \model{} variants on a filtered version of the public
RLHN-680K release~\citep{thakur2025rlhn}.\footnote{\url{https://huggingface.co/datasets/rlhn/rlhn-680K}} Each query in the source ships with one or more positive passages and a pool of mined negatives. We score every (query, candidate) pair with Qwen3-Embedding-8B~\citep{yang2025qwen3} (last-token pooling, $\ell_2$-normalized, max sequence length $512$), then apply two filters and a fixed-shape trim:
\begin{enumerate}[leftmargin=18pt,topsep=1pt,itemsep=0pt]
\item drop queries with fewer than $16$ candidates or fewer than $15$ negatives;
\item drop queries where the highest-scoring candidate is a negative, i.e.\ $\max \mathrm{score}(\text{pos}) \le \max \mathrm{score}(\text{neg})$;
\item retain the single best-scoring positive plus the top-$15$ highest-scoring negatives, giving exactly $16$ documents per query.
\end{enumerate}
The resulting post-filter, post-trim dataset has $522{,}487$ training samples drawn from $7$ source corpora (\cref{tab:rlhn-data-mix}); each sample is one query paired with its $16$ documents. Optimization steps are stratified per source: each batch contains queries from a single subset, so the gradient at any step reflects a single retrieval distribution. Documents are truncated to $T_{\text{doc}}{=}300$ tokens at training time. Token counts in \cref{tab:rlhn-data-mix} use the Qwen3-0.6B tokenizer; document-token averages are estimated from a stratified random sample of $8{,}000$ documents per subset (standard error $<0.5$ tokens at this $n$).

\begin{table}[h]
\centering\small
\caption{Composition of the post-filter, post-trim RLHN training mix. Source subsets carry their original RLHN names. Each training sample carries one query plus its $16$ documents.}
\label{tab:rlhn-data-mix}
% Auto-generated by testing/scripts/aggregate_rlhn_table.py
% Source: /data/rlhn_filtered (post-filter, post-trim RLHN-680K)
\begin{tabular}{lrrr}
\toprule
Source & \# training samples & Avg.\ query tok & Avg.\ doc tok \\
\midrule
MS\,MARCO & 368,961 & 7.0 & 84.6 \\
HotpotQA & 81,551 & 24.2 & 100.6 \\
FEVER & 28,561 & 11.7 & 265.1 \\
NQ & 27,962 & 10.5 & 146.5 \\
SCIDOCS-RR & 11,787 & 13.4 & 221.9 \\
FiQA & 2,822 & 13.9 & 225.7 \\
ArguAna & 843 & 251.9 & 209.9 \\
\midrule
Total & 522,487 & -- & -- \\
\bottomrule
\end{tabular}

\end{table}

\newpage
\section{Training Hyperparameters}
\label{app:training-hyperparams}

All \model{} variants in this paper share a single training recipe, fine-tuning Qwen/Qwen3-0.6B on the RLHN-filtered mix of \cref{app:rlhn-data} with the on-policy auxiliary loss of \cref{app:dagger}. The shared hyperparameters are listed in \cref{tab:hparams-shared}; the modification variants of \cref{sec:interventions} (\model-sink, \model-SSMax) layer length-aware attention modifications on top of this recipe and otherwise inherit every other knob. The variant-specific differences are listed in \cref{tab:hparams-diff}.

\paragraph{Note on the sink.} The per-layer sink scalars $b_L$ sit in their own parameter group with LR $1{\times}10^{-3}$ and zero weight decay (vs.\ $3{\times}10^{-5}$ / $0.01$ for the base weights): the higher LR compensates for the small gradients flowing through $\sigma(\mathrm{lse} - b_L)$ into a single scalar, and zero weight decay avoids pulling the gate toward an uninformative regime. The length signal is injected by sampling $N_{\text{eff}} \sim \log\mathcal{U}(N_0, 5\mathrm{k})$ per step ($N_0 {=} 128$ for our 8-GPU setup) and using $b_L + \alpha \log(N_{\text{eff}}/N_0)$ as the effective threshold, with the gate strength linearly ramped in over the first $\sim$2k steps. At evaluation the gate is disabled entirely: the learned $b_L$ stay in the checkpoint but do not affect attention. The mechanism only shapes what the rest of the model adapts to during training. As a check, we replayed the gate at evaluation with the trained thresholds and length schedule on MS\,MARCO: Recall@1 is $96.5$ / $75.5$ / $45.3$ / $21.3$ / $2.8$ at $N{=}500$ / $1$k / $2.5$k / $5$k / $10$k, against $96.5$ / $75.2$ / $45.2$ / $21.2$ / $2.5$ with the gate disabled (\cref{tab:promising-results}), so enabling it at test time is essentially a no-op.

\paragraph{Note on SSMax.} The per-layer scalars $s_L$ are initialized to $0.43$ following \cite{nakanishi2025ssmax} and trained in the shared parameter group --- no separate LR or weight decay. Length conditioning is explicit through $\log N$, so the same scaling applies unchanged at evaluation.

\begin{table}[h]
\centering
\small
\setlength{\tabcolsep}{4pt}
\caption{Variant-specific hyperparameters: differences between \model{}, \model-sink, and \model-SSMax. All other settings are shared (\cref{tab:hparams-shared}).}
\label{tab:hparams-diff}
\begin{tabular}{l c c c}
\toprule
& \textbf{\model{}} & \textbf{\model-sink} & \textbf{\model-SSMax} \\
\midrule
Layers modified                & -- & all $28$              & all $28$ \\
Per-layer parameter            & -- & sink scalar $b_L$     & scalar $s_L$ \\
Initialization                 & -- & $b_L^{(0)} = 14.0$    & $s_L^{(0)} = 0.43$ \\
Param-group LR                 & -- & $1{\times}10^{-3}$    & shared \\
Param-group weight decay       & -- & $0$                   & shared \\
Warmup / ramp (steps)          & -- & $500$ / $1500$        & --- \\
Length signal                  & -- & $b_L + \alpha \log\!\tfrac{N_{\text{eff}}}{N_0}$ & $s_L \log N$ \\
Strength $\alpha$              & -- & $1.0$                 & --- \\
Gated rows                     & -- & $\{0,1,2\}$           & --- \\
$N_{\text{eff}}$ at training   & -- & $\log\mathcal{U}(N_0, 5\mathrm{k})$ & --- \\
$N_0$                          & -- & $128$                 & --- \\
\bottomrule
\end{tabular}
\end{table}

\begin{table}[h]
\centering
\small
\caption{Shared training hyperparameters for all \model{} variants.}
\label{tab:hparams-shared}
\begin{tabular}{ll}
\toprule
\textbf{Group} & \textbf{Value} \\
\midrule
Base model               & Qwen/Qwen3-0.6B (bf16) \\
Distributed              & DDP, $8 \times$ NVIDIA A100 \\
Optimizer                & AdamW (fused), $\beta$ defaults \\
Base LR / weight decay   & $3{\times}10^{-5}$ / $0.01$ \\
LR schedule              & linear warmup (start factor $0.05$) over $1000$ steps \\
Per-GPU batch size       & $16$ queries (effective global $128$) \\
Epochs                   & $1$ pass over the RLHN-filtered triples \\
Documents per query      & $16$ (stratified from the top-$32$ RLHN candidates) \\
Document length cap      & $T_{\text{doc}} = 300$ tokens \\
Code width / scheme      & $4$ digits, sampled uniformly per training step \\
Query rotary offset      & $300$ \\
Loss                     & next-token CE on the $4$-digit code \\
\quad + on-policy aux.\ loss   & weight $\lambda = 0.5$ (\cref{app:dagger}) \\
\quad + KL distillation (auxiliary teacher) & EMA $\alpha = 0.95$ \\
Attention kernel         & block-sparse FlexAttention~\citep{dong2024flex} \\
\bottomrule
\end{tabular}
\end{table}

\newpage
\section{Training details for \texttt{Qwen3-dense}}
\label{app:embedding-model}

We fine-tune Qwen3-0.6B (FlashAttention-2, gradient checkpointing) into a last-token-pooled, $\ell_2$-normalized embedding model on the same RLHN-filtered mix used for \model{} training (\cref{app:rlhn-data}). Each training example pairs one query with $C{=}16$ teacher-scored candidates; queries are formatted with the same instruction prefix used at evaluation (``\textit{Given a web search query, retrieve relevant passages that answer the query}''), and only the last hidden state of each sequence is pooled. The loss is a contrastive term plus a teacher-distillation term:
\begin{equation}
\mathcal{L} \;=\; \mathcal{L}_{\text{CE}}
            + \lambda_{\text{KL}}\,\mathcal{L}_{\text{KL}},
\quad
\mathcal{L}_{\text{CE}} =
\mathrm{CE}\!\left(\tfrac{1}{\tau_s}\, q\, C_{\text{all}}^{\!\top},\; y\right),
\quad
\mathcal{L}_{\text{KL}} =
\mathrm{KL}\!\bigl(\sigma(s_t/\tau_t)\,\|\,\sigma(s_s/\tau_s)\bigr),
\end{equation}
where $C_{\text{all}}$ is the candidate matrix gathered (with gradient) across all DDP ranks, so each query is contrasted against $W \cdot B \cdot C$ documents in total; $y$ points at the gold candidate within the gathered block; and the KL term distills the teacher's softmax over the local 16 candidates into the student's softmax over the same 16. Hyperparameters are listed in \cref{tab:dense-v1-hparams}; we train a single epoch over the RLHN-filtered triples on $8{\times}$ NVIDIA A100, with cosine decay to $0.1\,\eta_0$ after a 500-step linear warmup, AdamW ($\beta$ defaults), and gradient clipping at $1.0$. The resulting checkpoint is the \texttt{Qwen3-dense-0.6B} dense baseline used throughout the paper.

\begin{table}[h]
\centering\small
\caption{\texttt{Qwen3-dense} training configuration.}
\label{tab:dense-v1-hparams}
\begin{tabular}{lll}
\toprule
\textbf{Component} & \textbf{Setting} & \textbf{Value} \\
\midrule
Backbone        & Qwen3-0.6B, bf16, FlashAttn-2, grad-ckpt    & --- \\
Pooling         & last-token, $\ell_2$-normalized              & --- \\
Query / doc len & max tokens                                  & $128$ / $300$ \\
Candidates      & per query (RLHN teacher-scored)             & $C{=}16$ \\
Negatives       & cross-rank gathered in-batch                & $W \cdot B \cdot C$ \\
Batch size      & per-GPU $\times$ GPUs                       & $32 \times 8$ \\
Optimizer       & AdamW, weight decay                         & $0.01$ \\
Learning rate   & peak $\eta_0$ / floor                       & $2{\times}10^{-5}$ / $0.1\,\eta_0$ \\
Schedule        & linear warmup $\to$ cosine decay            & $500$ steps warmup \\
Grad clip       & $\ell_2$                                    & $1.0$ \\
Student / teacher temperature & $\tau_s$ / $\tau_t$           & $0.02$ / $0.02$ \\
KL weight       & $\lambda_{\text{KL}}$                       & $0.5$ \\
Epochs          & over RLHN-filtered                          & $1$ \\
\bottomrule
\end{tabular}
\end{table}

\newpage
\section{On-Policy Auxiliary Loss}
\label{app:dagger}

\Cref{sec:training} introduces an on-policy auxiliary loss to mitigate exposure bias when decoding the four-digit code. The procedure is given in \cref{alg:onpolicy-aux}. For each query in the batch, we have $K$ candidate documents (1 positive, 15 hard negatives) with their full four-digit codes and Qwen3-Embedding-8B relevance scores, and the corpus has already been prefilled into the shared block-sparse cache (\cref{sec:training}).

\begin{algorithm}[h]
  \caption{On-policy auxiliary loss for ICR decoding (per minibatch).}
  \label{alg:onpolicy-aux}
  \begin{algorithmic}[1]
  \Require model $\pi_\theta$; batch of queries with prefilled document cache; per-query candidate codes $c^{(1)}, \dots, c^{(K)}$ (each a length-$L$ digit string) with teacher relevance scores $r^{(1)}, \dots, r^{(K)}$.
  \State Compute teacher weights $w_k \propto \exp(r^{(k)} / \tau_{\text{tch}})$ over the $K$ candidates.
  \Statex
  \State \textbf{(1) Rollout the model on its own.} Sample a code $\hat{a} = \hat{a}_1 \dots \hat{a}_L$ autoregressively from $\pi_\theta$ with gradients off.
  \Statex
  \State \textbf{(2) Build a per-position teacher.} For each digit position $t$:
  \State \quad Keep candidates whose first $t-1$ digits match the rollout: $\mathcal{M}_t = \{k : c^{(k)}_{<t} = \hat{a}_{<t}\}$.
  \State \quad Set the teacher distribution over the $10$ digits to the candidate-weighted vote:
  \[
    q_t(d) \;\propto\; \sum_{k \in \mathcal{M}_t} w_k \cdot \mathds{1}[c^{(k)}_t = d].
  \]
  \Statex
  \State \textbf{(3) Replay with gradients.} Re-decode the four answer positions conditioned on $\hat{a}_{<t}$ (the rollout's prefix, not the gold prefix) with gradients on, obtaining the student distributions $\pi_\theta(\cdot \mid \hat{a}_{<t})$.
  \Statex
  \State \textbf{Loss:} $\mathcal{L}_{\text{aux}} = \dfrac{1}{L} \sum_{t=1}^{L} \mathrm{CE}\!\bigl(q_t,\; \pi_\theta(\cdot \mid \hat{a}_{<t})\bigr)$, added to the total loss as $\mathcal{L} = \mathcal{L}_{\mathrm{CE}} + \lambda \mathcal{L}_{\text{aux}}$.
  \end{algorithmic}
\end{algorithm}

The rollout uses the model's own distribution rather than the gold code, so the prefixes that the four answer positions condition on are exactly the prefixes the model would visit at inference. The per-position teacher $q_t$ thus plays the role of a DAgger expert~\citep{ross2011dagger}: at every state visited along the model's own trajectory, it specifies what the model should have done. We ramp $\lambda$ in linearly from $0$ over the first warmup steps so early training is dominated by the standard teacher-forced cross-entropy.

\newpage
\section{Per-head softmax statistics at $\mathrm{L}19$}
\label{app:l19-signal-noise}

\Cref{sec:limitations} shows that at $\mathrm{L}19$ the gold document's share of the attention output collapses while the output's magnitude stays roughly constant. This appendix looks at the per-head attention statistics behind that result.

For a single head, the attention mass on the gold document is a sigmoid of the difference between two log-sum-exps, one over gold tokens and one over all other tokens: $m_{\mathcal{G}} = \sigma\bigl(\mathrm{lse}_{\mathcal{G}} - \mathrm{lse}_{\bar{\mathcal{G}}}\bigr)$. We track two quantities at $\mathrm{L}19$ as $N$ grows (\cref{tab:l19-signal-noise}): gold's largest attention score $s^{\max}_{\mathcal{G}}$, which dominates $\mathrm{lse}_{\mathcal{G}}$, and the gap $\mathrm{lse}_{\bar{\mathcal{G}}} - \mathrm{lse}_{\mathcal{G}}$. We report medians over the six $\mathrm{L}19$ heads that carry the retrieval signal (heads 4, 5, 7, 9, 12, and 13). Six other heads at this layer act as attention sinks and give the gold document a mass of $1/N$ at every $N$, so including them would hide most of the collapse.

For \model{}, gold's largest score falls from $25.31$ to $13.86$ between $N{=}500$ and $N{=}10\mathrm{k}$, and the gap grows from $-1.59$ to $+6.34$. As a result, gold's attention mass falls from $0.82$ to $0.0018$, a drop of $7.94$ nats in log-odds, or a factor of about $2{,}800$. A corpus that is $20\times$ larger would by itself account for $\log 20 = 3.00$ nats of this~\citep{velickovic2024softmax,barbero2024glasses}. The remaining $4.94$ nats come from gold's score falling relative to the other documents.\footnote{A document's code depends on its position in the sorted corpus, so the same passage receives a different code, and therefore different keys after the code tokens, at each $N$. Assigning codes by document identity instead changes $s^{\max}_{\mathcal{G}}$ by at most $0.17$ at any $N$ and leaves the drop in gold's attention mass unchanged ($0.824$ to $0.0018$).}

\Cref{tab:l19-signal-noise} also includes the sink and SSMax variants from \cref{sec:interventions}. The sink behaves almost like \model{}, which matches its small effect on recall. SSMax raises gold's score as $N$ grows ($+7.06$ instead of $-11.45$), but it raises the scores of the other documents even more (\cref{tab:l19-ssmax}). Gold's attention mass at $N{=}10\mathrm{k}$ is therefore $11\times$ higher than for \model{} but still below $2\%$, which is consistent with the partial recovery in recall.

\begin{table}[h]
\centering\small
\caption{Per-head attention statistics at $\mathrm{L}19$ on MS\,MARCO ($n{=}100$), taking the median over the six retrieval heads within each query and then over queries. $s^{\max}_{\mathcal{G}}$ is gold's largest pre-softmax score, the gap $\mathrm{lse}_{\bar{\mathcal{G}}} - \mathrm{lse}_{\mathcal{G}}$ is lower when it favors gold, and gold mass is gold's share of the post-softmax attention.}
\label{tab:l19-signal-noise}
\begin{tabular}{l r rrr}
\toprule
Checkpoint & $N$ & $s^{\max}_{\mathcal{G}}$ & $\mathrm{lse}_{\bar{\mathcal{G}}} - \mathrm{lse}_{\mathcal{G}}$ & Gold mass \\
\midrule
\model{}       & $500$   & $25.31$ & $-1.59$ & $0.823$ \\
               & $1\mathrm{k}$   & $24.71$ & $+0.09$ & $0.478$ \\
               & $2.5\mathrm{k}$ & $23.35$ & $+2.19$ & $0.102$ \\
               & $5\mathrm{k}$   & $19.85$ & $+4.06$ & $0.018$ \\
               & $10\mathrm{k}$  & $13.86$ & $+6.34$ & $0.002$ \\
\midrule
\model-sink    & $500$   & $26.42$ & $-1.82$ & $0.853$ \\
               & $1\mathrm{k}$   & $26.02$ & $-0.09$ & $0.521$ \\
               & $2.5\mathrm{k}$ & $24.25$ & $+2.24$ & $0.100$ \\
               & $5\mathrm{k}$   & $21.99$ & $+3.55$ & $0.028$ \\
               & $10\mathrm{k}$  & $16.86$ & $+5.74$ & $0.003$ \\
\midrule
\model-SSMax   & $500$   & $33.97$ & $-2.60$ & $0.915$ \\
               & $1\mathrm{k}$   & $35.77$ & $-0.63$ & $0.625$ \\
               & $2.5\mathrm{k}$ & $37.90$ & $+1.81$ & $0.142$ \\
               & $5\mathrm{k}$   & $38.99$ & $+2.80$ & $0.057$ \\
               & $10\mathrm{k}$  & $41.03$ & $+4.04$ & $0.019$ \\
\bottomrule
\end{tabular}
\end{table}

\begin{table}[h]
\centering\small
\caption{Change in gold's attention at $\mathrm{L}19$ from $N{=}500$ to $N{=}10\mathrm{k}$ for each model, from the same measurements as \cref{tab:l19-signal-noise}. The odds drop is the fall in gold's log-odds against the rest of the corpus, in nats and as a factor. A $20\times$ larger corpus alone would account for $3.00$ nats ($20\times$).}
\label{tab:l19-ssmax}
\begin{tabular}{l rr rr r}
\toprule
Checkpoint & \multicolumn{2}{c}{Odds drop} & $\Delta s^{\max}_{\mathcal{G}}$ & $\Delta\,\mathrm{lse}_{\bar{\mathcal{G}}}$ & Gold mass at $10\mathrm{k}$ \\
\midrule
\model{}      & $7.94$ & $2{,}800\times$ & $-11.45$ & $-2.57$  & $0.0018$ \\
\model-sink   & $7.56$ & $1{,}929\times$ & $-9.56$  & $-1.45$  & $0.0034$ \\
\model-SSMax  & $6.64$ & $764\times$     & $+7.06$  & $+12.64$ & $0.0195$ \\
\bottomrule
\end{tabular}
\end{table}

\newpage
\section{What fills the attention output at large $N$}
\label{app:l19-content}

\Cref{tab:l19-attnout} shows that the $\mathrm{L}19$ output stops being driven by the gold document at large $N$. To see what replaces it, we split the output into three pieces that together cover all document tokens: the shared first token of every document (\texttt{Doc}), the gold document's content tokens, and the other documents' content tokens (\cref{tab:l19-3way}). We report each piece's length as a share of the total output length, as a median over $100$ queries. The pieces point in different directions, so the shares do not sum to one, and together they reconstruct the full output to within about $0.3\%$.

The shared first token absorbs much of the attention mass in many heads, so it could in principle come to dominate the output at large $N$. It does not. Its share grows only from $0.21$ to $0.35$, while the share from other documents' content grows from $0.28$ to $0.87$, so the layer is writing the content of the wrong documents rather than a generic signal. The gold document's own first token contributes at most $0.0004$ of the output, so it does not inflate the gold share in \cref{tab:l19-attnout}.

As $N$ grows, the output also becomes less specific to the query. The median cosine similarity between the $\mathrm{L}19$ outputs of different queries rises from $0.47$ at $N{=}500$ to $0.90$ at $N{=}10\mathrm{k}$, while the gold-driven piece stays query-specific ($0.30$ to $0.36$) but shrinks. The model does not fall back to a single default answer, however. At $N{=}10\mathrm{k}$, greedy decoding produces $78$ distinct codes for $100$ queries, and for $N \le 5\mathrm{k}$, beam search produces $385$ to $398$ distinct top-1 answers for $400$ queries. The model retrieves the wrong documents rather than emitting a fixed answer.

\begin{table}[h]
\centering\small
\caption{Three-way split of the $\mathrm{L}19$ attention output on \model{}, MS\,MARCO, $100$ queries. Each share is a piece's length divided by the total output length, computed per query and then taken as the median.}
\label{tab:l19-3way}
\begin{tabular}{l r rrr}
\toprule
$N$ & $\|a_{19}\|$ & shared \texttt{Doc} token & gold content & non-gold content \\
\midrule
$500$           & $47.48$ & $0.210$ & $0.905$ & $0.282$ \\
$1\mathrm{k}$   & $45.36$ & $0.218$ & $0.742$ & $0.426$ \\
$2.5\mathrm{k}$ & $43.03$ & $0.227$ & $0.193$ & $0.798$ \\
$5\mathrm{k}$   & $36.90$ & $0.280$ & $0.060$ & $0.871$ \\
$10\mathrm{k}$  & $30.27$ & $0.347$ & $0.007$ & $0.872$ \\
\bottomrule
\end{tabular}
\end{table}

\newpage
\section{Equivalence of the additive-sink and sigmoid-gate forms of attention}
\label{app:sink-gate-proof}

We prove that the additive-sink softmax of \cref{eq:sink} is exactly equivalent to multiplying the standard softmax of \cref{eq:qk-scores} by a sigmoid gate.

\begin{proposition}
\label{prop:sink-gate}
Fix a layer $L$ and head $h$. Let $s = (s^{h}_{L,1}, \dots, s^{h}_{L,T}) \in \mathbb{R}^T$ be the pre-softmax logits at the $d_1$-emission position, $b_L \in \mathbb{R}$ the learned sink scalar, $\alpha^{h}_{L,t}$ the standard softmax weight from \cref{eq:qk-scores}, and $\tilde{\alpha}^{h}_{L,t}$ the additive-sink weight from \cref{eq:sink}. Define the layer-and-head-wise gate
\[
g^{h}_{L} \;=\; \sigma\!\bigl(\mathrm{lse}(s) - b_L\bigr), \qquad \mathrm{lse}(s) \;=\; \log\!\sum_{t=1}^{T} \exp(s^{h}_{L,t}), \qquad \sigma(x) = \frac{1}{1 + e^{-x}}.
\]
Then for every $t \in \{1, \dots, T\}$,
\[
\tilde{\alpha}^{h}_{L,t} \;=\; \alpha^{h}_{L,t} \cdot g^{h}_{L}.
\]
\end{proposition}

\begin{proof}
Let $Z = \sum_{t'=1}^{T} \exp(s^{h}_{L,t'}) = \exp(\mathrm{lse}(s))$ denote the standard softmax normalizer. By \cref{eq:qk-scores}, $\alpha^{h}_{L,t} = \exp(s^{h}_{L,t}) / Z$. By \cref{eq:sink}, the additive-sink weight is
\[
\tilde{\alpha}^{h}_{L,t} \;=\; \frac{\exp(s^{h}_{L,t})}{Z + \exp(b_L)}.
\]
Multiply numerator and denominator by $1/Z$:
\[
\tilde{\alpha}^{h}_{L,t} \;=\; \frac{\exp(s^{h}_{L,t})/Z}{1 + \exp(b_L)/Z} \;=\; \alpha^{h}_{L,t} \cdot \frac{1}{1 + \exp(b_L - \mathrm{lse}(s))} \;=\; \alpha^{h}_{L,t} \cdot \sigma\!\bigl(\mathrm{lse}(s) - b_L\bigr) \;=\; \alpha^{h}_{L,t} \cdot g^{h}_{L},
\]
where the third equality uses $\sigma(x) = 1/(1 + e^{-x})$ with $x = \mathrm{lse}(s) - b_L$. \qedhere
\end{proof}

\paragraph{Interpretation.} The gate $g^{h}_{L} \in (0,1)$ depends only on the per-(layer, head) comparison between $\mathrm{lse}(s)$ --- a measure of how much pre-softmax mass the layer's logits already concentrate --- and the learned threshold $b_L$. When the layer's logits are sharp and concentrated relative to $b_L$ (large $\mathrm{lse}$), $g^{h}_{L} \to 1$ and \cref{eq:sink} reduces to the standard softmax. When the logits are diffuse and many tokens contribute small amounts (small $\mathrm{lse}$ relative to $b_L$), $g^{h}_{L} \to 0$ and the layer's contribution to the residual update $a_L$ is suppressed multiplicatively. This is the mechanism by which the sink ``costs'' a noisy layer its write to the residual stream without altering the relative ranking of the per-token weights.

\newpage
\section{Top-$B$ Routing Recall}
\label{app:topb-routing}

The top-$B$ routing of \cref{sec:interventions} hinges on the assumption that the router can keep the gold document inside a small top-$B$ shortlist with high probability across our checkpoints. We measure exactly that: for $B \in \{32, 64, 128, 256, 512, 1024\}$ we report the fraction of queries whose gold document survives the QK-MaxSim router at layer~$16$, the default routing configuration used in \cref{sec:interventions}. Evaluation uses $N{=}10{,}000$ MS\,MARCO documents and $n{=}400$ queries (the same query slice for all three checkpoints). Recall@$B$ is derived per query from the gold document's rank under the router's score relative to the full $N{=}10{,}000$ ranking.

Across all three checkpoints the gold document is retained for at least $93.5\%$ of queries at the default $B{=}256$ (\model: $0.962$; \model-sink: $0.962$; \model-SSMax: $0.935$), and the curves flatten between $B{=}256$ and $B{=}512$ --- going from $B{=}256$ to $B{=}1024$ recovers an additional $\le 5$ percentage points on every checkpoint (\cref{fig:topb-routing-recall}). This is consistent with $B{=}256$ being a practical routing budget on this corpus: the remaining miss rate is dominated by the long tail of routing errors (gold ranked beyond the top-$1024$), which is unaffected by any reasonable enlargement of $B$. \model-SSMax's lower curve at small $B$ ($0.762$ at $B{=}32$ vs.\ $0.838$ for \model) is consistent with SSMax placing slightly less probability mass on the gold document at the routing layer, but the gap closes by $B{=}512$.

\begin{figure}[h]
\centering
% Auto-generated by testing/scripts/aggregate_topb_routing_curve.py
% Source: testing/results/topb_routing_recall_curve/raw/  (dataset=msmarco_N10000)
% recall@B vs B for stage-1 QK_MAXSIM (qk_expand_k) routing at L16, n=400 queries.
\definecolor{cBlock52}{HTML}{0072B2}
\definecolor{cBlock60}{HTML}{D55E00}
\definecolor{cBlock67}{HTML}{009E73}
\begin{tikzpicture}[trim axis left, trim axis right]
  \begin{axis}[
      width=0.92\linewidth,
      height=4.0cm,
      scale only axis,
      xmode=log, log basis x=2,
      xmin=32, xmax=1024,
      ymin=0.0, ymax=1.05,
      xlabel={Top-$B$ routing budget},
      ylabel={Recall@$B$ (gold survives)},
      xlabel style={font=\small, yshift=2pt},
      ylabel style={font=\small, yshift=-4pt},
      x tick label style={font=\footnotesize},
      y tick label style={font=\footnotesize},
      xtick={32,64,128,256,512,1024},
      xticklabels={32,64,128,256,512,1024},
      ytick={0,0.25,0.5,0.75,1.0},
      grid=both,
      grid style={line width=.2pt, draw=gray!20},
      major grid style={line width=.3pt, draw=gray!35},
      legend style={
        font=\footnotesize,
        at={(0.98,0.05)}, anchor=south east,
        draw=gray!40, fill=white, fill opacity=0.9, text opacity=1,
        inner sep=2pt, row sep=-1pt,
      },
      legend cell align=left,
      enlarge x limits=false,
    ]
    \addplot[color=cBlock52, mark=*, mark size=1.6pt, line width=1.0pt] coordinates {(32,0.8375) (64,0.9050) (128,0.9425) (256,0.9625) (512,0.9775) (1024,0.9875)};
    \addlegendentry{\model}
    \addplot[color=cBlock60, mark=square*, mark size=1.6pt, line width=1.0pt] coordinates {(32,0.8350) (64,0.9050) (128,0.9425) (256,0.9625) (512,0.9800) (1024,0.9850)};
    \addlegendentry{\model-sink}
    \addplot[color=cBlock67, mark=triangle*, mark size=1.6pt, line width=1.0pt] coordinates {(32,0.7625) (64,0.8500) (128,0.9100) (256,0.9350) (512,0.9725) (1024,0.9850)};
    \addlegendentry{\model-SSMax}
    \draw[dashed, gray, line width=0.5pt] (axis cs:256,0.0) -- (axis cs:256,1.05);
    \node[font=\scriptsize, anchor=south, text=gray!70] at (axis cs:256,1.05) {$B=256$};
  \end{axis}
\end{tikzpicture}
\caption{Routing recall@$B$ for QK-MaxSim routing at layer~$16$ on MS\,MARCO with $N{=}10{,}000$ and $n{=}400$ queries. Dashed line marks the default $B{=}256$ used throughout the paper.}
\label{fig:topb-routing-recall}
\end{figure}

\newpage
\section{Beam Search and Recall@5}
\label{app:beam-recall5}

The main text reports Recall@1 over the single highest-probability four-digit code (\cref{sec:evaluation}). Here we (i) specify the decoding procedure used to turn the model's per-digit distributions into a ranked list of document codes, and (ii) report the corresponding \emph{Recall@5} across the full length-generalization sweep, so that the main-text Recall@1 conclusions can be checked at a wider cutoff. 

\subsection{Decoding}
\label{app:decoding}
A query is answered by a digit-by-digit beam search over the full $N$-document corpus (context-parallel across GPUs), returning a ranked list of complete codes. Recall@$k$ is computed over this ranking. 

\subsection{Digit-by-digit beam search}
\label{app:beam-algo}
Constrained beam search over a tree of identifier tokens is the standard decoder for generative and in-context retrieval, and we follow that practice rather than introduce a new procedure~\citep{tay2022dsi,wang2022nci,gupta2025blockrank}. Because each identifier here is exactly four digits drawn from a $10$-symbol vocabulary, the search tree has depth $4$ and branching factor $10$, and the beam is over digit strings rather than free-form text. \Cref{alg:beam} shows the routine; our default configuration is beam width $B{=}25$, per-step pruning $k_{\text{step}}{=}5$, and return depth $k_{\text{ret}}{=}5$ (so Recall@5 reads the top five returned codes).

\begin{algorithm}[t]
  \caption{Digit-by-digit beam search for a single query.}
  \label{alg:beam}
  \begin{algorithmic}[1]
  \Require model $\pi_\theta$ with the document corpus prefilled into its (block-sparse) KV cache; digit token ids $\mathcal{D}=\{\mathrm{id}(0),\dots,\mathrm{id}(9)\}$; code length $L{=}4$; beam width $B$; per-step prune $k_{\text{step}}$; return depth $k_{\text{ret}}$.
  \State Append the query and forward once; let $\ell_1 = \log\mathrm{softmax}$ of the next-token logits restricted to $\mathcal{D}$.
  \State Initialize $B$ beams from the top-$\min(B,10)$ first digits, with beam scores $=\ell_1$ of the chosen digits; broadcast the KV cache across the $B$ beams.
  \For{$t = 2,\dots,L$}
    \State Forward all live beams one step; let $\ell_t^{(b)}$ be the digit log-probs for beam $b$.
    \State Candidate scores: $\mathrm{score}(b,d) = \mathrm{beam\_score}(b) + \ell_t^{(b)}(d)$.
    \State Per beam, keep only its top-$k_{\text{step}}$ digits (set the rest to $-\infty$).
    \State Flatten over $(b,d)$ and keep the global top-$B$ survivors; record each survivor's parent beam and digit.
    \State Reorder the KV cache by parent index; append the chosen digit to each surviving beam.
  \EndFor
  \State Sort the $B$ completed beams by score; \Return the top-$k_{\text{ret}}$ digit strings.
  \end{algorithmic}
\end{algorithm}

\subsection{Recall@5 across the length-generalization sweep}
\label{app:recall5-tables}
\Cref{tab:recall5} reports Recall@5 for the same variants, datasets, and corpus sizes as the main-text Recall@1 (\cref{fig:length-gen,tab:promising-results}). The qualitative picture is identical to Recall@1: the position-coded variant and the no-modification model collapse by $N{=}5$k--$10$k on every dataset, while SSMax, top-$B$ routing, and their composition hold up across the full sweep. Recall@5 stays low wherever Recall@1 collapses and rises with it wherever the modifications hold (e.g.\ MS\,MARCO \model{} at $N{=}10$k: R@1 $0.2$, R@5 $0.2$; \model-routing $18.8 \to 47.5$), confirming that the main-text collapse is not an artifact of the top-1 cutoff. For HotpotQA, whose queries have two golds, we also report Recall@2 in \cref{tab:hotpotqa-r2} for completeness. % src: testing/results/MULTISAMPLE_TIER1_V2.md

\begin{table}[h]
\centering
\myfontsize
\setlength{\tabcolsep}{4pt}
\caption{\textbf{Recall@5} ($\times 100$) vs.\ corpus size $N$, mirroring \cref{tab:promising-results} ($n{=}400$). NQ caps at $N{=}8{,}607$. HotpotQA has two golds per query, so its columns report true recall (the fraction of the two golds that land in the top $5$); Natural Questions and MS\,MARCO are effectively single-gold, where this matches the hit-based metric used in the main body. % src: testing/results/MULTISAMPLE_TIER1_V2.md
}
\label{tab:recall5}
\resizebox{\linewidth}{!}{%
\begin{tabular}{l rrrrr rrrrr rrrrr}
\toprule
& \multicolumn{5}{c}{Natural Questions} & \multicolumn{5}{c}{MS\,MARCO} & \multicolumn{5}{c}{HotpotQA} \\
\cmidrule(lr){2-6} \cmidrule(lr){7-11} \cmidrule(lr){12-16}
$N$ & 0.5k & 1k & 2.5k & 5k & 8.6k & 0.5k & 1k & 2.5k & 5k & 10k & 0.5k & 1k & 2.5k & 5k & 10k \\
\midrule
\model               & 96.0 & 91.7 & 62.9 & \phantom{0}9.5 & \phantom{0}0.2 & 99.5 & 98.0 & 76.2 & 38.2 & \phantom{0}0.2 & 70.6 & 67.9 & 47.8 & 11.0 & \phantom{0}0.0 \\
\model-position      & 96.7 & 90.7 & 39.4 & \phantom{0}2.8 & \phantom{0}0.2 & 99.2 & 96.5 & 67.0 & \phantom{0}7.8 & \phantom{0}1.0 & 67.9 & 64.4 & 34.6 & \phantom{0}1.9 & \phantom{0}0.0 \\
\model-offpolicy     & 95.5 & 90.7 & 56.9 & \phantom{0}6.0 & \phantom{0}0.2 & 99.8 & 97.0 & 77.8 & 29.2 & \phantom{0}0.5 & 68.2 & 66.4 & 43.0 & \phantom{0}6.1 & \phantom{0}0.0 \\
\midrule
\model-sink          & 95.5 & 90.5 & 65.7 & 21.8 & \phantom{0}2.3 & 99.8 & 97.8 & 76.8 & 45.2 & \phantom{0}9.0 & 73.1 & 70.6 & 46.2 & 13.6 & \phantom{0}0.0 \\
\model-SSMax         & 96.7 & 93.5 & 80.2 & 70.2 & 56.9 & 99.5 & 97.8 & 83.2 & 63.5 & 44.5 & 69.8 & 68.5 & 59.4 & 52.4 & 41.9 \\
\model-routing       & 98.2 & 95.2 & 85.2 & 71.2 & 60.2 & 99.5 & 98.5 & 83.0 & 66.8 & 47.5 & 75.0 & 74.4 & 67.8 & 60.1 & 53.6 \\
\model-SSMax-routing & 97.0 & 94.5 & 79.7 & 73.9 & 61.2 & 99.2 & 97.8 & 80.5 & 65.8 & 46.0 & 74.2 & 72.5 & 67.0 & 59.6 & 52.4 \\
\bottomrule
\end{tabular}%
}
\end{table}

\begin{table}[h]
\centering
\small
\setlength{\tabcolsep}{5pt}
\caption{HotpotQA \textbf{Recall@2} ($\times 100$, true recall = fraction of a query's two golds in the top $2$, $n{=}400$), for completeness alongside the HotpotQA Recall@5 column of \cref{tab:recall5} (same runs). % src: testing/results/E0_E33_E34_tier1_v2_lg/
}
\label{tab:hotpotqa-r2}
\begin{tabular}{l rrrrr}
\toprule
Method & 0.5k & 1k & 2.5k & 5k & 10k \\
\midrule
\model               & 63.4 & 61.2 & 39.8 & \phantom{0}8.9 & \phantom{0}0.0 \\
\model-position      & 61.2 & 58.0 & 29.6 & \phantom{0}1.2 & \phantom{0}0.0 \\
\model-offpolicy     & 61.2 & 58.5 & 36.5 & \phantom{0}3.8 & \phantom{0}0.0 \\
\midrule
\model-sink          & 65.0 & 61.1 & 38.5 & \phantom{0}9.5 & \phantom{0}0.0 \\
\model-SSMax         & 62.4 & 61.1 & 52.9 & 44.8 & 35.2 \\
\model-routing       & 68.1 & 67.2 & 58.6 & 52.1 & 46.0 \\
\model-SSMax-routing & 65.8 & 64.0 & 57.8 & 50.9 & 45.4 \\
\bottomrule
\end{tabular}
\end{table}

\newpage
\section{LIMIT worked example}
\label{app:limit-example}

\Cref{fig:limit-example} shows a worked LIMIT-small example for the lexical retrieval task evaluated in \cref{sec:limit-ood}.

\begin{figure}[h]
\centering
\small
\fbox{\begin{minipage}{0.95\linewidth}
\ttfamily\raggedright
\# LIMIT-small corpus: one short biography per person (46 documents).\\[2pt]
{[}BOS{]}Doc 6359: Geneva Durben likes Quokkas, River Otters, Tapirs, Asymmetry, Snow Leopards, \dots, Joshua Trees, Pansies, Soy Sauce, Cards Against Humanity and Elm Trees. (Doc 6359){[}EOS{]}\\[2pt]
{[}BOS{]}Doc 9841: \dots\ (a different person, whose list does not contain Joshua Trees) \dots\ (Doc 9841){[}EOS{]}\\[6pt]
\# Query block: appended after all 46 documents, at RoPE position 300.\\[2pt]
{[}BOS{]}Instruct: Given a web search query, retrieve relevant passages that answer the query\\
Query: Who likes Joshua Trees?\\
Answer: \underline{6\;3\;5\;9}
\end{minipage}}
\caption{Worked LIMIT-small example. Relevance is lexical: the gold document is a person whose preference list contains the queried item (\texttt{Joshua Trees}), among the $46$ documents. Each query has two gold documents (we show one); either code is a correct answer. The query and gold text are from the evaluated corpus. % src: testing/results/unfiled/eval_output_limit_small_52_debug.txt
}
\label{fig:limit-example}
\end{figure}

\subsection{Recall@2 and Recall@5}
\label{app:limit-proprecall}
The main-text \cref{tab:limit} reports Recall@1 under the metric used throughout the body (a query counts if either gold is in the top 1). Here we additionally report standard Recall@$k$: the fraction of a query's two golds that appear in the top $k$, at $k{=}2$ (\cref{tab:limit-r2}) and $k{=}5$ (\cref{tab:limit-r5}). The picture is the same: SSMax+routing holds up across the sweep while \model{} and \model-sink collapse by $N{=}2.5$k.

\begin{table}[h]
\centering
\small
\setlength{\tabcolsep}{5pt}
\caption{LIMIT \textbf{Recall@2} (fraction of a query's two golds in the top 2, $n{=}1{,}000$) vs.\ corpus size $N$, for the same runs as \cref{tab:limit}. The dense baseline's $N{=}10{,}000$ cell is not available. % src: testing/results/E45_limit_lengthgen/ + dense rerun
}
\label{tab:limit-r2}
\begin{tabular}{l l rrrrrr}
\toprule
Method & Scoring & $46$ & $500$ & $1000$ & $2500$ & $5000$ & $10000$ \\
\midrule
\model{}                      & ICR beam      & $0.286$ & $0.069$ & $0.015$ & $0.000$ & $0.000$ & $0.001$ \\
\model-sink                   & ICR beam      & $0.198$ & $0.038$ & $0.016$ & $0.002$ & $0.000$ & $0.000$ \\
\model-SSMax                  & ICR beam      & $0.344$ & $0.177$ & $0.130$ & $0.070$ & $0.048$ & $0.013$ \\
\textbf{\model-SSMax-routing} & ICR beam      & $\mathbf{0.344}$ & $\mathbf{0.197}$ & $\mathbf{0.177}$ & $\mathbf{0.170}$ & $\mathbf{0.133}$ & $\mathbf{0.093}$ \\
\midrule
\texttt{Qwen3-dense}          & pooled cosine & $0.160$ & $0.063$ & $0.045$ & $0.036$ & $0.025$ & --- \\
Random chance                 & ---           & $0.043$ & $0.004$ & $0.002$ & $0.001$ & $0.000$ & $0.000$ \\
\bottomrule
\end{tabular}
\end{table}

\begin{table}[h]
\centering
\small
\setlength{\tabcolsep}{5pt}
\caption{LIMIT \textbf{Recall@5} (fraction of a query's two golds in the top 5, $n{=}1{,}000$) vs.\ corpus size $N$, for the same runs as \cref{tab:limit}. The dense baseline's $N{=}10{,}000$ cell is not available. % src: testing/results/E45_limit_lengthgen/ + dense rerun
}
\label{tab:limit-r5}
\begin{tabular}{l l rrrrrr}
\toprule
Method & Scoring & $46$ & $500$ & $1000$ & $2500$ & $5000$ & $10000$ \\
\midrule
\model{}                      & ICR beam      & $0.436$ & $0.104$ & $0.027$ & $0.000$ & $0.000$ & $0.003$ \\
\model-sink                   & ICR beam      & $0.340$ & $0.070$ & $0.032$ & $0.004$ & $0.000$ & $0.000$ \\
\model-SSMax                  & ICR beam      & $0.488$ & $0.231$ & $0.164$ & $0.095$ & $0.067$ & $0.019$ \\
\textbf{\model-SSMax-routing} & ICR beam      & $\mathbf{0.488}$ & $\mathbf{0.258}$ & $\mathbf{0.241}$ & $\mathbf{0.225}$ & $\mathbf{0.182}$ & $\mathbf{0.135}$ \\
\midrule
\texttt{Qwen3-dense}          & pooled cosine & $0.300$ & $0.098$ & $0.066$ & $0.050$ & $0.039$ & --- \\
Random chance                 & ---           & $0.109$ & $0.010$ & $0.005$ & $0.002$ & $0.001$ & $0.001$ \\
\bottomrule
\end{tabular}
\end{table}

\newpage
\section{OBLIQ results}
\label{app:obliq}

We additionally evaluate on OBLIQ~\citep{tchuindjo2026obliq}, a benchmark of ``\emph{oblique}'' retrieval, where relevance is indirect rather than lexical or topical. We use three of its tasks. \textbf{Math} and \textbf{Writing} are \emph{analogues} tasks: the query is a passage (a competition mathematics problem, or a paragraph of prose), and the gold documents are other passages that share an abstract reasoning pattern or argument with it despite entirely different surface content. \textbf{Twitter} is a \emph{descriptive} task: the query is a natural-language description of a class of posts, and the gold documents are tweets matching that description (see \cref{fig:obliq-example} for one example per task). Corpus size $N$ is set per task by feasibility: Math uses the full corpus, while Twitter and Writing use gold-preserving subsamples (drawn from the full $72$k- and $10{,}388$-document corpora); per-dataset document counts, query counts, and average lengths are given in \cref{tab:obliq-data}. We apply the OBLIQ self-match exclusion (each query's own source passage is dropped from its ranking). As in \cref{sec:limit-ood} we evaluate the four \model{} variants and the same-backbone pooled dense retriever, and additionally a long-context generative-retrieval baseline, \texttt{MSA-4B}~\citep{chen2026msa}; for ICR the ranked list is the beam search output. Unlike the main body, here we report standard Recall@$k$: the fraction of a query's gold documents that appear in the top $k$. Recall is inversely related to the number of golds per query --- a top-$k$ list can hold at most $k$ of them --- so with OBLIQ's many golds (a mean of ${\sim}9$--$13$ per query) the values are correspondingly small.

\begin{table}[h]
\centering
\small
\caption{OBLIQ dataset statistics for the evaluated subsets. Lengths are in Qwen3-0.6B tokens (mean over documents, and over queries that have a gold). Documents are truncated to $300$ tokens at prefill, which affects mainly Writing. Twitter and Writing corpora are gold-preserving subsamples of the full $72$k- and $10{,}388$-document corpora. % src: ~/datasets/obliq_\{math,twitter\_10k,writing\_4k\}
}
\label{tab:obliq-data}
\begin{tabular}{l r r r r}
\toprule
 & & & \multicolumn{2}{c}{Avg.\ tokens} \\
\cmidrule(lr){4-5}
Dataset & Documents & Queries & Doc & Query \\
\midrule
Math (analogue)       & $3{,}507$  & $151$ & $116$ & $181$ \\
Twitter (descriptive) & $10{,}000$ & $281$ & $52$  & $53$  \\
Writing (analogue)    & $4{,}000$  & $512$ & $498$ & $653$ \\
\bottomrule
\end{tabular}
\end{table}

\begin{figure}[h]
\centering
\small
\fbox{\begin{minipage}{0.95\linewidth}
\raggedright
\textbf{Math (analogues).} \emph{Instruction:} Given a mathematical problem statement, retrieve passages that present analogous problems or techniques.\\[2pt]
\emph{Query:} Let $(M,d)$ be a nonempty complete metric space and $S:M\to M$ with $S^2$ a strict contraction; show that $S$ has a unique fixed point. \dots\\[2pt]
\emph{Gold:} Let $U$ be a nonempty bounded open set in $\mathbb{R}^n$; \dots\ show that there is an affine transformation carrying $U$ to the unit ball.\\[6pt]
\textbf{Writing (analogues).} \emph{Instruction:} Given a passage of writing, retrieve other passages with analogous ideas, arguments, or stylistic techniques.\\[2pt]
\emph{Query:} ``We see some good news in alignment -- as models become more capable, they are also more aligned \dots\ but still far from the reliability required in high-stakes applications.''\\[2pt]
\emph{Gold:} ``Leonard Shelby, the protagonist of Christopher Nolan's film \emph{Memento}, suffers from anterograde amnesia \dots\ he uses notes, photos, and tattoos to communicate facts to his future self.''\\[6pt]
\textbf{Twitter (descriptive).} \emph{Instruction:} Given a description of a tweet, retrieve tweets that match the description.\\[2pt]
\emph{Query:} Find tweets where users implicitly mock feel-good branding and token measures from tech moguls and politicians while actual wars intensify.\\[2pt]
\emph{Gold:} ``@user\ `but we decided to send 8000 helmets to the Strait of Hormuz to help defend the free world'\,''
\end{minipage}}
\caption{Worked OBLIQ examples, one per task. Math and Writing are \emph{analogues} tasks (the gold shares an abstract pattern with the query despite different surface content); Twitter is a \emph{descriptive} task. Query and gold text are from the evaluated corpora, abbreviated; the Twitter handle is anonymized. % src: dianetc/OBLIQ-Bench
}
\label{fig:obliq-example}
\end{figure}

\begin{table}[h]
\centering
\small
\setlength{\tabcolsep}{4pt}
\caption{OBLIQ results: Recall@1 / Recall@5 (the fraction of a query's gold documents that appear in the top $k$; with many golds per query these values are small), with the OBLIQ self-match exclusion applied. \texttt{MSA-4B} emits ${\sim}1$ citation per query, so its R@5 is close to its R@1. % src: testing/results/obliq_full_sweep/ [artifact:testing/results/obliq_full_sweep]
}
\label{tab:obliq}
\begin{tabular}{l l cc cc cc}
\toprule
 & & \multicolumn{2}{c}{Math ($N{=}3.5$k)} & \multicolumn{2}{c}{Twitter ($N{=}10$k)} & \multicolumn{2}{c}{Writing ($N{=}4$k)} \\
\cmidrule(lr){3-4}\cmidrule(lr){5-6}\cmidrule(lr){7-8}
Method & Scoring & R@1 & R@5 & R@1 & R@5 & R@1 & R@5 \\
\midrule
\model{}                       & ICR beam      & $0.000$ & $0.000$ & $0.000$ & $0.000$ & $0.002$ & $0.004$ \\
\model-sink                    & ICR beam      & $0.000$ & $0.001$ & $0.000$ & $0.001$ & $0.000$ & $0.003$ \\
\model-SSMax                   & ICR beam      & $0.000$ & $0.011$ & $0.001$ & $0.004$ & $0.000$ & $0.005$ \\
\model-SSMax-routing           & ICR beam      & $0.007$ & $0.029$ & $0.000$ & $0.003$ & $0.001$ & $0.006$ \\
\midrule
\texttt{Qwen3-dense}           & pooled cosine & $0.014$ & $0.057$ & $0.000$ & $0.004$ & $0.009$ & $0.025$ \\
\texttt{MSA-4B}                & gen.\ citations & $0.013$ & $0.027$ & $0.002$ & $0.008$ & $0.004$ & $0.008$ \\
\bottomrule
\end{tabular}
\end{table}

\newpage
\section{Generality beyond Qwen3 and \model{}}
\label{app:generality}

We test whether the collapse of \cref{sec:collapse} and its mechanism (\cref{sec:limitations}) are specific to Qwen3-0.6B and our training recipe, in two ways: an inference-only evaluation of an off-the-shelf proprietary long-context LM, and a \model{} trained on a second backbone family.

\paragraph{Off-the-shelf long-context LM.} We evaluate Gemini-3-flash on the MS\,MARCO setup of \cref{sec:evaluation} with no task training, prompting it to emit the codes of the relevant documents ($400$ queries). Its Recall@1 falls from $95.0$ at $N{=}500$ to $16.2$ at $N{=}10{,}000$ (\cref{tab:generality}): the same corpus-scale degradation appears without our architecture, data, or training. Because the model is only reachable through an API, the mechanistic analyses below are not possible for it.

\paragraph{Second backbone: SmolLM2-360M.} We train \model{} with the same recipe on SmolLM2-360M (\texttt{LlamaForCausalLM}, $32$ layers), with the learning rate raised to $10^{-4}$. It starts lower than the Qwen3 model and collapses earlier (\cref{tab:generality}), reaching $\le 1\%$ Recall@1 by $N{=}5{,}000$ on all three datasets. On LIMIT its Recall@1 is $0.117$ at $N{=}46$ and at most $0.002$ from $N{=}500$ onward; on OBLIQ it is at most $0.001$ (Recall@5 at most $0.003$) on every task.

\begin{table}[h]
\centering\small
\setlength{\tabcolsep}{4pt}
\caption{Recall@1 ($\times 100$) for an off-the-shelf long-context LM and for \model{} on two backbones. $400$ queries per corpus size. Gemini-3-flash was evaluated on MS\,MARCO only. NQ's largest corpus is $N{=}8{,}607$.}
\label{tab:generality}
\resizebox{\linewidth}{!}{%
\begin{tabular}{l rrrrr rrrrr rrrrr}
\toprule
& \multicolumn{5}{c}{Natural Questions} & \multicolumn{5}{c}{MS\,MARCO} & \multicolumn{5}{c}{HotpotQA} \\
\cmidrule(lr){2-6} \cmidrule(lr){7-11} \cmidrule(lr){12-16}
$N$ & 0.5k & 1k & 2.5k & 5k & 8.6k & 0.5k & 1k & 2.5k & 5k & 10k & 0.5k & 1k & 2.5k & 5k & 10k \\
\midrule
Gemini-3-flash (no training) & --- & --- & --- & --- & --- & 95.0 & 72.8 & 52.0 & 35.0 & 16.2 & --- & --- & --- & --- & --- \\
\model{} (Qwen3-0.6B)        & 92.2 & 77.7 & 43.9 & 4.8 & 0.2 & 95.8 & 75.2 & 43.8 & 18.8 & 0.2 & 97.0 & 95.2 & 64.2 & 13.0 & 0.5 \\
\model{} (SmolLM2-360M)      & 58.9 & 23.8 & 1.0 & 0.0 & 0.2 & 85.3 & 55.5 & 12.5 & 1.0 & 0.0 & 70.0 & 59.5 & 2.0 & 0.2 & 0.0 \\
\bottomrule
\end{tabular}%
}
\end{table}

\paragraph{The mechanism on SmolLM2.} We repeat the analyses of \cref{sec:limitations} on the SmolLM2 model. The same two-band structure appears, shifted deeper: at $N{=}500$, $R^{\text{sum}}_L$ peaks at $0.59$--$0.69$ across $\mathrm{L}17$--$\mathrm{L}23$ while first-digit accuracy is still at chance ($\approx 0.10$), and first-digit accuracy lifts off at $\mathrm{L}23{\to}\mathrm{L}24$ ($0.10 \to 0.19$, reaching $0.33$ by $\mathrm{L}31$) ($400$ queries). As $N$ grows (\cref{fig:layer-roles-smollm}), $R^{\text{sum}}_L$ falls to $0.00$ at every layer by $N{=}10{,}000$, and $R^{\text{any}}_L$ decays more slowly but is also near zero at $N{=}10{,}000$ (at most $0.15$, at $\mathrm{L}17$; $0.09$ at $\mathrm{L}18$). The same shared-prefix sink behavior described in \cref{sec:layer-roles} appears in this model. Finally, the attention-output decomposition at $\mathrm{L}23$, the decode onset, shows the same swap as Qwen3 at $\mathrm{L}19$ (\cref{tab:goldshare-smollm}): the layer keeps writing into the residual at a comparable magnitude while the gold-driven share falls from $0.64$ to $0.00$.

\begin{figure}[h]
\centering
% Per-layer R^sum_L and R^any_L vs L -- SmolLM2-360M (smollm52_lr1e4),
% MS MARCO draw0, n=100, doc_agg=max, layers 0-31, N in {500, 1k, 2.5k, 5k, 10k}.
% src: testing/results/unfiled/tie_probe_rany_full/
\definecolor{cN500}{HTML}{FDE725}     % yellow-green (small N)
\definecolor{cN1k}{HTML}{7AD151}      % light green
\definecolor{cN2p5k}{HTML}{22A884}    % teal
\definecolor{cN5k}{HTML}{2A788E}      % steel blue
\definecolor{cN10k}{HTML}{414487}     % deep purple (large N)
\begin{tikzpicture}
  \begin{groupplot}[
      group style={
        group size=2 by 1,
        horizontal sep=1.4cm,
      },
      width=0.38\linewidth,
      height=3.2cm,
      scale only axis,
      xmin=-0.5, xmax=31.5,
      ymin=-0.05, ymax=1.10,
      xtick={0,4,8,12,16,20,24,28},
      ytick={0,0.25,0.5,0.75,1.0},
      x tick label style={font=\footnotesize},
      y tick label style={font=\footnotesize},
      xlabel style={font=\small, yshift=2pt},
      ylabel style={font=\small, yshift=-2pt},
      grid=both,
      grid style={line width=.2pt, draw=gray!20},
      major grid style={line width=.3pt, draw=gray!35},
      every axis plot/.append style={line width=0.8pt},
      title style={font=\small, yshift=-2pt},
      enlarge x limits=false,
    ]
    %====================== Left panel: R^sum ======================
    \nextgroupplot[
        title={(a) $R^{\text{sum}}_L$ vs $L$},
        xlabel={Layer $L$},
        ylabel={$R^{\text{sum}}_L$},
        legend to name=layerrolestiecorrsmolllegend,
        legend style={
          font=\footnotesize,
          draw=none, fill=none,
          /tikz/every even column/.append style={column sep=0.4cm},
        },
        legend columns=5,
        legend cell align=left,
    ]
    \addplot[color=cN500, mark=*, mark size=1.2pt] coordinates {
      (0,0.00) (1,0.00) (2,0.00) (3,0.00) (4,0.00) (5,0.00) (6,0.00) (7,0.00) (8,0.00) (9,0.00)
      (10,0.00) (11,0.00) (12,0.01) (13,0.00) (14,0.00) (15,0.01) (16,0.06) (17,0.71) (18,0.55) (19,0.00)
      (20,0.34) (21,0.01) (22,0.00) (23,0.74) (24,0.09) (25,0.05) (26,0.17) (27,0.02) (28,0.03) (29,0.01)
      (30,0.00) (31,0.00)
    };
    \addlegendentry{$N{=}500$}
    \addplot[color=cN1k, mark=*, mark size=1.2pt] coordinates {
      (0,0.00) (1,0.00) (2,0.00) (3,0.00) (4,0.00) (5,0.00) (6,0.00) (7,0.00) (8,0.00) (9,0.00)
      (10,0.00) (11,0.00) (12,0.00) (13,0.00) (14,0.00) (15,0.00) (16,0.00) (17,0.11) (18,0.14) (19,0.00)
      (20,0.18) (21,0.00) (22,0.00) (23,0.32) (24,0.01) (25,0.01) (26,0.11) (27,0.03) (28,0.02) (29,0.01)
      (30,0.00) (31,0.00)
    };
    \addlegendentry{$N{=}1\mathrm{k}$}
    \addplot[color=cN2p5k, mark=*, mark size=1.2pt] coordinates {
      (0,0.00) (1,0.00) (2,0.00) (3,0.00) (4,0.00) (5,0.00) (6,0.00) (7,0.00) (8,0.00) (9,0.00)
      (10,0.00) (11,0.00) (12,0.00) (13,0.00) (14,0.00) (15,0.00) (16,0.00) (17,0.00) (18,0.00) (19,0.00)
      (20,0.01) (21,0.00) (22,0.00) (23,0.06) (24,0.00) (25,0.00) (26,0.02) (27,0.03) (28,0.01) (29,0.01)
      (30,0.00) (31,0.00)
    };
    \addlegendentry{$N{=}2.5\mathrm{k}$}
    \addplot[color=cN5k, mark=*, mark size=1.2pt] coordinates {
      (0,0.00) (1,0.00) (2,0.00) (3,0.00) (4,0.00) (5,0.00) (6,0.00) (7,0.00) (8,0.00) (9,0.00)
      (10,0.00) (11,0.00) (12,0.00) (13,0.00) (14,0.00) (15,0.00) (16,0.00) (17,0.00) (18,0.00) (19,0.00)
      (20,0.00) (21,0.00) (22,0.00) (23,0.01) (24,0.00) (25,0.00) (26,0.00) (27,0.01) (28,0.00) (29,0.01)
      (30,0.00) (31,0.00)
    };
    \addlegendentry{$N{=}5\mathrm{k}$}
    \addplot[color=cN10k, mark=*, mark size=1.2pt] coordinates {
      (0,0.00) (1,0.00) (2,0.00) (3,0.00) (4,0.00) (5,0.00) (6,0.00) (7,0.00) (8,0.00) (9,0.00)
      (10,0.00) (11,0.00) (12,0.00) (13,0.00) (14,0.00) (15,0.00) (16,0.00) (17,0.00) (18,0.00) (19,0.00)
      (20,0.00) (21,0.00) (22,0.00) (23,0.00) (24,0.00) (25,0.00) (26,0.00) (27,0.00) (28,0.01) (29,0.00)
      (30,0.00) (31,0.00)
    };
    \addlegendentry{$N{=}10\mathrm{k}$}
    %====================== Right panel: tie-corrected R^any ======================
    \nextgroupplot[
        title={(b) $R^{\text{any}}_L$ vs $L$},
        xlabel={Layer $L$},
        ylabel={$R^{\text{any}}_L$},
    ]
    \addplot[color=cN500, mark=triangle*, mark size=1.3pt] coordinates {
      (0,0.00) (1,0.03) (2,0.00) (3,0.00) (4,0.05) (5,0.01) (6,0.03) (7,0.08) (8,0.00) (9,0.53)
      (10,0.12) (11,0.07) (12,0.39) (13,0.02) (14,0.28) (15,0.14) (16,0.30) (17,0.95) (18,0.79) (19,0.01)
      (20,0.81) (21,0.16) (22,0.18) (23,0.90) (24,0.51) (25,0.62) (26,0.48) (27,0.13) (28,0.09) (29,0.05)
      (30,0.05) (31,0.02)
    };
    \addplot[color=cN1k, mark=triangle*, mark size=1.3pt] coordinates {
      (0,0.00) (1,0.00) (2,0.00) (3,0.00) (4,0.01) (5,0.00) (6,0.00) (7,0.01) (8,0.00) (9,0.40)
      (10,0.06) (11,0.05) (12,0.22) (13,0.00) (14,0.22) (15,0.07) (16,0.06) (17,0.75) (18,0.58) (19,0.00)
      (20,0.58) (21,0.09) (22,0.07) (23,0.68) (24,0.29) (25,0.33) (26,0.14) (27,0.08) (28,0.03) (29,0.06)
      (30,0.06) (31,0.02)
    };
    \addplot[color=cN2p5k, mark=triangle*, mark size=1.3pt] coordinates {
      (0,0.00) (1,0.00) (2,0.00) (3,0.00) (4,0.00) (5,0.00) (6,0.00) (7,0.00) (8,0.00) (9,0.15)
      (10,0.01) (11,0.00) (12,0.07) (13,0.01) (14,0.13) (15,0.00) (16,0.01) (17,0.45) (18,0.34) (19,0.00)
      (20,0.20) (21,0.01) (22,0.01) (23,0.31) (24,0.09) (25,0.08) (26,0.07) (27,0.03) (28,0.02) (29,0.03)
      (30,0.01) (31,0.00)
    };
    \addplot[color=cN5k, mark=triangle*, mark size=1.3pt] coordinates {
      (0,0.00) (1,0.00) (2,0.00) (3,0.00) (4,0.00) (5,0.00) (6,0.00) (7,0.00) (8,0.00) (9,0.08)
      (10,0.01) (11,0.00) (12,0.03) (13,0.01) (14,0.06) (15,0.00) (16,0.00) (17,0.23) (18,0.16) (19,0.00)
      (20,0.05) (21,0.00) (22,0.01) (23,0.07) (24,0.00) (25,0.04) (26,0.01) (27,0.00) (28,0.01) (29,0.01)
      (30,0.01) (31,0.00)
    };
    \addplot[color=cN10k, mark=triangle*, mark size=1.3pt] coordinates {
      (0,0.00) (1,0.00) (2,0.00) (3,0.00) (4,0.00) (5,0.00) (6,0.00) (7,0.00) (8,0.00) (9,0.04)
      (10,0.00) (11,0.00) (12,0.00) (13,0.00) (14,0.03) (15,0.01) (16,0.00) (17,0.15) (18,0.09) (19,0.00)
      (20,0.00) (21,0.00) (22,0.01) (23,0.02) (24,0.00) (25,0.01) (26,0.00) (27,0.00) (28,0.01) (29,0.01)
      (30,0.00) (31,0.00)
    };
  \end{groupplot}
  % Shared legend, centered below both panels
  \node[anchor=north, inner sep=2pt] (leg)
    at ([yshift=-32pt]$(group c1r1.south east)!0.5!(group c2r1.south west)$)
    {\pgfplotslegendfromname{layerrolestiecorrsmolllegend}};
\end{tikzpicture}
\caption{Per-layer recall vs.~$L$ across $N \in \{500, 1000, 2500, 5000, 10000\}$ on \model{} trained on SmolLM2-360M, MS\,MARCO, $100$ queries. \textbf{(a)} $R^{\text{sum}}_L$; \textbf{(b)} $R^{\text{any}}_L$. Compare \cref{fig:layer-roles-n10k}.}
\label{fig:layer-roles-smollm}
\end{figure}

\begin{table}[h]
\centering\small
\caption{Attention-output decomposition at $\mathrm{L}23$ on \model{} trained on SmolLM2-360M, MS\,MARCO, $400$ queries; norms are medians over queries and the share is the median of per-query ratios. Compare \cref{tab:l19-attnout}.}
\label{tab:goldshare-smollm}
\begin{tabular}{l rrrr}
\toprule
$N$ & $\|a_{23}^{\mathcal{G}}\|$ & $\|a_{23}^{\bar{\mathcal{G}}}\|$ & $\|a_{23}\|$ & $\|a_{23}^{\mathcal{G}}\| / \|a_{23}\|$ \\
\midrule
$500$          & $35.87$ & $32.84$ & $60.16$ & $0.64$ \\
$1\mathrm{k}$  & $15.93$ & $40.06$ & $57.59$ & $0.27$ \\
$2.5\mathrm{k}$ & $1.29$ & $40.71$ & $43.50$ & $0.03$ \\
$5\mathrm{k}$  & $0.07$ & $35.94$ & $36.08$ & $0.00$ \\
$10\mathrm{k}$ & $0.01$ & $34.50$ & $34.51$ & $0.00$ \\
\bottomrule
\end{tabular}
\end{table}

\newpage
\section{General capability after retrieval training}
\label{app:general-capability}

\model{} is trained only to emit four-digit document codes, so we measure how much general capability the backbone retains, using lm-eval-harness (0-shot) under standard causal attention (\cref{tab:general-capability}). Retrieval-only fine-tuning costs $6.6$ points of MMLU and $9.0$ of HellaSwag on Qwen3-0.6B, and $9.0$ of HellaSwag on SmolLM2-360M (whose stock MMLU is already at chance, so it cannot show a drop). Both models stay well above chance on HellaSwag. The sink variant matches \model{}; we did not evaluate the SSMax and routing variants.

We read this as the cost of single-objective fine-tuning with no general-purpose or instruction data in the mixture, not as evidence that retrieval and generation are incompatible. Our goal was to isolate whether a decoder can learn corpus-scale retrieval at all; a model meant to do both would mix general-purpose data into training, which we leave to future work.

\begin{table}[h]
\centering\small
\caption{General capability (lm-eval-harness, 0-shot) before and after retrieval training. Random chance is $25.0$ on both tasks.}
\label{tab:general-capability}
\begin{tabular}{l rr}
\toprule
Checkpoint & MMLU & HellaSwag (acc\_norm) \\
\midrule
Qwen3-0.6B (stock)              & 40.2 & 47.3 \\
\model{}                        & 33.6 & 38.3 \\
\model-sink                     & 33.5 & 37.9 \\
\midrule
SmolLM2-360M (stock)            & 25.4 & 56.4 \\
\model{} (SmolLM2-360M)         & 24.8 & 47.4 \\
\bottomrule
\end{tabular}
\end{table}

\section{Training data and document truncation}
\label{app:data-ablation}

\paragraph{How much comes from RLHN filtering?} We compare two checkpoints with an identical recipe (architecture, auxiliary loss, randomized codes, hyperparameters, seed) that differ only in training data: the RLHN-filtered mix of \cref{app:rlhn-data} and our original MS\,MARCO mix with false negatives left in (\cref{tab:rlhn-ablation}). The filtered mix helps little through $N{=}2{,}500$ (at most $1.5$ points), roughly doubles recall at $N{=}5{,}000$, and makes no difference at $N{=}10{,}000$, where both collapse. Most of \model{}'s performance therefore does not come from data filtering. The comparison with \texttt{Qwen3-dense} is data-controlled in any case, since it trains on the same RLHN mix.

\begin{table}[h]
\centering\small
\caption{MS\,MARCO Recall@1 ($\times 100$) for two \model{} checkpoints that differ only in training data; $400$ queries.}
\label{tab:rlhn-ablation}
\begin{tabular}{l rrrrr}
\toprule
$N$ & 500 & 1k & 2.5k & 5k & 10k \\
\midrule
RLHN-filtered mix (\model{})   & 95.75 & 75.25 & 43.75 & 18.75 & 0.25 \\
Original MS\,MARCO mix         & 95.25 & 74.00 & 42.25 & \phantom{0}9.00 & 0.00 \\
\bottomrule
\end{tabular}
\end{table}

\paragraph{How often does the 300-token cap truncate?} Each document block is capped at $300$ tokens, of which $18$ are the fixed identifier wrapper, so a document is truncated once its text exceeds $282$ tokens. In the training mixture this happens for $0.04\%$ of MS\,MARCO $(q,d)$ pairs, $0.69\%$ of HotpotQA pairs, and $0.04\%$ of NQ pairs (estimated from a $25$k-document sample per source), so truncation plays essentially no role for these datasets. The cap is a property of our evaluation budget rather than of the method: block-sparse prefill encodes each document independently, so longer blocks only increase prefill cost.

\newpage
\section{System cost}
\label{app:system-cost}
  
In-context retrieval moves the index into the KV cache, so its costs differ in kind from dense retrieval. We report decode latency, which we measured, and memory, which follows from the architecture.

\paragraph{Latency.} \Cref{tab:latency} reports the per-query decode time on MS\,MARCO (\model{}, $400$ queries, $8\times$A100-80GB, a single timed run per cell). Dense decode grows roughly linearly with the corpus, from $0.37$\,s at $N{=}500$ to $4.51$\,s at $N{=}10\mathrm{k}$ ($12.3\times$). With top-$B{=}256$ routing, only the shortlisted documents enter attention from $\mathrm{L}17$ onward, and latency grows from $0.49$\,s to $0.84$\,s ($1.7\times$ over a $20\times$ larger corpus), a $5.4\times$ speedup over dense decode at $N{=}10\mathrm{k}$. Routed decode is therefore much less sensitive to corpus size than dense decode, though not flat; at $N{=}500$ the routing pass is pure overhead. The corpus prefill is computed once and reused across queries, so its cost ($5.5$--$11.9$\,s per corpus in our runs) is amortized rather than paid per query.

\paragraph{Memory.} Memory is the main practical cost. Qwen3-0.6B has $28$ layers with $8$ KV heads of dimension $128$, so its bf16 KV cache takes $112$\,KiB per token. At $N{=}10\mathrm{k}$ MS\,MARCO documents ($\approx 0.95$M tokens) this is $\approx 101$\,GiB, which we shard across $8$ GPUs by context parallelism. A dense index over the same corpus with $1{,}024$-dimensional bf16 embeddings takes $\approx 20$\,MB, about three orders of magnitude less. Dense retrieval also benefits from sublinear approximate nearest-neighbor search, whereas attention compares the query against every key. Routing is a coarse step toward avoiding that comparison; attention that is simultaneously sparse, differentiable, and GPU-efficient, together with KV-cache compression, is in our view the central systems problem for this direction.

\begin{table}[ht]
\centering\small
\caption{Per-query decode latency (seconds) on MS\,MARCO, \model{}, $400$ queries, $8\times$A100-80GB, single run per cell. Routing uses top-$B{=}256$ at $\mathrm{L}16$. Recall@1 for each cell matches \cref{tab:promising-results}.}
\label{tab:latency}
\begin{tabular}{l rrrrr}
\toprule
$N$ & 500 & 1k & 2.5k & 5k & 10k \\
\midrule
Dense attention        & 0.37 & 0.57 & 1.20 & 2.34 & 4.51 \\
Top-$B$ routing        & 0.49 & 0.53 & 0.57 & 0.65 & 0.84 \\
Speedup                & $0.75\times$ & $1.07\times$ & $2.11\times$ & $3.61\times$ & $5.35\times$ \\
\bottomrule
\end{tabular}
\end{table}

%%%%%%%%%%%%%%%%%%%%%%%%%%%%%%%%%%%%%%%%%%%%%%%%%%%%%%%%%%%%

%\newpage
%\input{checklist.tex}

\end{document}